\documentclass{article}

\usepackage{eqnarray,amsmath}

 \usepackage{multirow}

\newcommand{\bbE}{\mathbb{E}}
\newcommand{\var}{\mathrm{Var}}

\newcommand{\softmax}{\mathrm{Softmax}}
\newcommand{\sigmoid}{\mathrm{sigmoid}}
\newcommand{\gelu}{\mathrm{GELU}}

\newcommand{\dint}{\mathrm{d}}

\newcommand{\dboijn}{\Delta \beta_{1ij}^{n}}

\newcommand{\daijn}{\Delta a_{ij}^{n}}
\newcommand{\dbijn}{\Delta b_{ij}^{n}}
\newcommand{\deijn}{\Delta \eta_{ij}^{n}}

\newcommand{\boin}{\beta_{1i}^n}
\newcommand{\bzin}{\beta_{0i}^n}
\newcommand{\ain}{a_i^n}
\newcommand{\bin}{b_i^n}
\newcommand{\ein}{\eta_i^n}

\newcommand{\boj}{\beta_{1j}^*}
\newcommand{\bzj}{\beta_{0j}^*}
\newcommand{\aj}{a_j^*}
\newcommand{\bj}{b_j^*}
\newcommand{\ej}{\eta_j^*}

\newcommand{\boi}{\beta_{1i}^*}
\newcommand{\bzi}{\beta_{0i}^*}

\newcommand{\zerod}{{0}_d}

\newcommand{\normf}[1]{\|#1\|_{L^2(\mu)}}

\DeclareMathOperator*{\argmin}{arg\,min}

\usepackage{microtype}
\usepackage{graphicx}
\usepackage{subcaption}
\usepackage{booktabs} 
\usepackage{ragged2e}

\usepackage{hyperref}

\usepackage[accepted]{icml2024}

\usepackage{amsmath}
\usepackage{amssymb}
\usepackage{mathtools}
\usepackage{amsthm}

\usepackage[capitalize,noabbrev]{cleveref}

\theoremstyle{plain}
\newtheorem{theorem}{Theorem}[section]
\newtheorem{proposition}[theorem]{Proposition}
\newtheorem{lemma}[theorem]{Lemma}

\theoremstyle{definition}
\newtheorem{definition}[theorem]{Definition}

\theoremstyle{remark}

\usepackage[textsize=tiny]{todonotes}

\icmltitlerunning{On Least Square Estimation in Softmax Gating Mixture of Experts}

\begin{document}

\twocolumn[
\icmltitle{On Least Square Estimation in Softmax Gating Mixture of Experts}



\icmlsetsymbol{equal}{*}

\begin{icmlauthorlist}
\icmlauthor{Huy Nguyen}{austin}
\icmlauthor{Nhat Ho}{austin,equal}
\icmlauthor{Alessandro Rinaldo}{austin,equal}
\end{icmlauthorlist}

\icmlaffiliation{austin}{Department of Statistics and Data Sciences, The University of Texas at Austin, USA}

\icmlcorrespondingauthor{Huy Nguyen}{huynm@utexas.edu}

\icmlkeywords{Machine Learning, ICML}

\vskip 0.3in
]



\printAffiliationsAndNotice{\icmlEqualContribution} 

\begin{abstract}
    Mixture of experts (MoE) model is a statistical machine learning design that aggregates multiple expert networks using a softmax gating function in order to form a more intricate and expressive model. Despite being commonly used in several applications owing to their scalability, the mathematical and statistical properties of MoE models are complex and difficult to analyze. As a result, previous theoretical works have primarily focused on probabilistic MoE models by imposing the impractical assumption that the data are generated from a Gaussian MoE model. In this work, we investigate the performance of the least squares estimators (LSE) under a deterministic MoE model where the data are sampled according to a  regression model, a setting that  has remained largely unexplored. We establish a condition called strong identifiability to characterize the convergence behavior of various types of expert functions. We demonstrate that the rates for estimating strongly identifiable experts, namely the widely used feed-forward networks with activation functions $\sigmoid(\cdot)$ and $\tanh(\cdot)$, are substantially faster than those of polynomial experts, which we show to exhibit a surprising slow estimation rate. Our findings have important practical implications for expert selection.
\end{abstract}

\section{Introduction}
\label{sec:introduction}
Softmax gating mixture of experts (MoE) is introduced by \cite{Jacob_Jordan-1991,Jordan-1994} as a generalization of classical mixture models \cite{Mclachlan-1988,Lindsay-1995} based on an adaptive gating mechanism. More concretely, the MoE model is a weighted sum of expert functions associated with input-dependent weights. Here, each expert is either a regression function \cite{deveaux1989linear,faria2010regression} or a classifier \cite{chen2022theory,nguyen_general_2023} that specializes in smaller parts of a larger problem. Meanwhile, the softmax gate is responsible for determining the weight of each expert's output. If one expert consistently outperforms others in some domains of the input space, the softmax gate will assign it a larger weight in those domains. Thanks to its flexibility and adaptability, there has been a surge of interest in using the softmax gating MoE models in several fields, namely large language models \cite{jiang2024mixtral,puigcerver2024sparse,zhou2023brainformers,Du_Glam_MoE,fedus2021switch}, computer vision \cite{Riquelme2021scalingvision,liang_m3vit_2022,Ruiz_Vision_MoE}, multi-task learning \cite{hazimeh_dselect_k_2021,gupta2022sparsely} and reinforcement learning \cite{chow_mixture_expert_2023}. In those applications, each expert plays an essential role in handling one or a few subproblems. As a consequence, it is of practical importance to study the problem of expert estimation, which can be solved indirectly via the parameter estimation problem. 

Despite its widespread use in practice, the theory for parameter estimation of the MoE model has not been fully comprehended. From a probabilistic perspective, \cite{ho2022gaussian} studied the convergence of maximum likelihood estimation under an input-independent gating Gaussian MoE, which admits the following set-up:

\textbf{Set-up of a Gaussian MoE model.} An i.i.d sample $(X_1,Y_1),\ldots,(X_n,Y_n)$ are assumed to be drawn from a softmax gating Gaussian MoE model whose conditional density function $p_{G_*}(y|x)$ is of the form
$$\sum_{i = 1}^{k_{*}} \frac{\exp((\beta_{1i}^{*})^{\top} x + \beta_{0i}^{*})}{\sum_{j = 1}^{k_{*}} \exp((\beta_{1j}^{*})^{\top} x + \beta_{0j}^{*})}\cdot \pi(y|h(x,\eta^*_i), \nu_{i}^{*}),$$ where $\pi(\cdot|\mu,\nu)$ denotes a Gaussian density function with mean $\mu$ and variance $\nu$, and $h(\cdot,\eta)$ stands for a mean expert function. Additionally, $G_{*} := \sum_{i = 1}^{k_{*}} \exp(\beta_{0i}^{*}) \delta_{(\beta_{1i}^{*},\eta^*_i,\nu^*_i)}$ stands for the \emph{mixing measure,} a weighted sum of Dirac measures $\delta$, with unknown parameters  $(\beta^*_{0i},\beta^*_{1i},\eta^*_{i},\nu^*_{i})$.

By assuming that the data were generated from that model, they demonstrated that the density estimation rate was parametric on the sample size, while the parameter estimation rates depended on the algebraic independence between expert functions. Subsequently, \cite{nguyen2023demystifying} and \cite{nguyen2024gaussian} also considered the Gaussian MoE models but equipped with a softmax gate and a Gaussian gate, respectively, both of which vary with the input values. Owing to an interaction among gating and expert parameters, they showed that the rates for estimating parameters were determined by the solvability of some systems of polynomial equations. Additionally, \cite{makkuva19gridlock} also designed provably consistent algorithms for learning parameters of the softmax gating Gaussian MoE. Next, \cite{nguyen2024statistical} investigated a Top-K sparse gating Gaussian MoE model \cite{shazeer2017topk,fedus2022review}, which activated only one or a few experts for each input. Their findings suggested that turning on exactly one expert per input would remove the interaction of gating parameters with those of experts, and therefore, accelerate the parameter estimation rates. 

While the theoretical advances in MoE modeling from recent years have been remarkable, a persistent and significant limitation of all existing contributions in the literature is the reliance on the strong assumption of a well-specified model, namely that the data are sampled from a (say, Gaussian) MoE model. This is of course, an unrealistic assumption that does not reflect real-world data \cite{li2023sparse,pham2024competesmoe}. Unfortunately, very little is known about the statistical properties of MoE models in mis-specified but more realistic regression settings.

In this paper, we partially address this gap by introducing and analyzing a more general regression framework for MoE models in which, conditionally on the features, the response variables are not sampled from a gated MoE but are instead noisy realization of an unknown and deterministic gated MoE-type regression function, as described next.

{\bf Set-up.} We assume that an i.i.d. sample of size $n$: $(X_{1}, Y_{1}), (X_{2}, Y_{2}), \ldots, (X_{n}, Y_{n})$ in $\mathbb{R}^d\times\mathbb{R}$ is generated according to the model
\begin{align}
    Y_{i} = f_{G_{*}}(X_{i}) + \varepsilon_{i}, \quad i=1,\ldots,n, \label{eq:moe_regression_model}
\end{align}
where $\varepsilon_{1}, \ldots, \varepsilon_{n}$ are independent Gaussian noise variables such that $\bbE[{\varepsilon_{i}}|X_i] = 0$ and $\var(\varepsilon_{i}|X_i) = \nu$ for all $1 \leq i \leq n$. Note that, the Gaussian assumption is just for the simplicity of proof argument. Furthermore, we assume that $X_{1}, \ldots, X_{n}$ are i.i.d. samples from some probability distribution $\mu$. Above, the  regression function $f_{G_{*}}(\cdot)$ takes the form of a softmax gating MoE with $k_*$ experts, namely
\begin{align}
    \label{eq:standard_MoE}
    f_{G_{*}}(x) := \sum_{i=1}^{k_*} \frac{\exp((\beta^*_{1i})^{\top}x+\beta^*_{0i})}{\sum_{j=1}^{k_*}\exp((\beta^*_{1j})^{\top}x+\beta^*_{0j})}\cdot h(x,\eta^*_i),
\end{align}
where $(\beta^*_{0i},\beta^*_{1i},\eta^*_{i})_{i=1}^{k_*}$ are unknown parameters in $\mathbb{R}\times\mathbb{R}^d\times\mathbb{R}^q$ and $G_{*} := \sum_{i = 1}^{k_{*}} \exp(\beta_{0i}^{*}) \delta_{(\beta_{1i}^{*},\eta^*_i)}$ denotes the associated \emph{mixing measure}. The function $h(x,\eta)$ is known as {\it the expert function,} which we assumed to be of parametric form. We will consider general expert functions as well as the widely used ridge expert functions $h(x;(a,b)) = \sigma(a^\top x + b)$, compositions of a non-linear {\it activation function} $\sigma(\cdot)$ with an affine function. See Section \ref{sec:background} below for further restrictions on the model. In practice, since the true number of experts $k_*$ is unknown, it is customary to fit a softmax gating MoE model of the form \eqref{eq:standard_MoE} with up to $k>k_*$ experts, where $k$ is a given threshold. We call this setting an \emph{over-specified} setting.

In order to estimate the  ground-truth parameters $(\beta^*_{0i},\beta^*_{1i},\eta^*_{i})_{i=1}^{k_*}$ in the above model, we can no longer rely on maximum likelihood estimation. Instead we will deploy the computationally efficient and popular least squares method  \citep[see, e.g.,][]{vandeGeer-00}. Formally, the mixing measure is estimated with  
\begin{align}
    \label{eq:least_squared_estimator}
    \widehat{G}_n:=\argmin_{G\in\mathcal{G}_{k}(\Theta)}\sum_{i=1}^{n}\Big(Y_i-f_{G}(X_i)\Big)^2,
\end{align}
where $\mathcal{G}_{k}(\Theta):=\{G=\sum_{i=1}^{k'}\exp(\beta_{0i})\delta_{(\beta_{1i},\eta_{i})}:1\leq k'\leq k, \  (\beta_{0i},\beta_{1i},\eta_{i})\in\Theta\}$ is the set of all mixing measures with at most $k$ components. The goal of this paper is to investigate the convergence properties of estimator $\widehat{G}_n$ in fixed-dimensional setting. To the best of our knowledge, this is the first statistical analysis of the least squares estimation under the MoE models, as previous works \cite{mendes2011convergence,nguyen2023demystifying} focus  on maximum likelihood methods.

\textbf{Challenges.} We highlight two subtle major challenges in analyzing the regression model \eqref{eq:standard_MoE}, which require  the formulation of novel identifiability conditions and new techniques. To the best of our knowledge, these issues have not been noted before in the regression literature.

\textbf{(C.1) Expert characterization.} In our analysis (which conforms to the latest approaches to MoE modeling), we represent the discrepancy $f_{\widehat{G}_n}(\cdot)-f_{G_*}(\cdot)$ between the estimated and true regression function as a weighted sum of linearly independent terms by applying Taylor expansions to the function $x \mapsto F(x;\beta_1,\eta):=\exp(\beta_{1}^{\top}x)h(x,\eta)$. In order to guarantee good convergence rates, it is necessary that the function $F$ and its derivatives are linearly independent (in the space of squared-integrable functions of the features $X$). This property will be ensured by formulating novel and non-trivial algebraic condition on the expert functions, which we refer to as {\it strong identifiability}. The derivation of that condition requires us to adopt new proof techniques since those in previous works \cite{nguyen2023demystifying,nguyen2024statistical} apply only for linear experts.

\textbf{(C.2) Singularity of polynomial experts.} An instance of expert functions  that does not satisfy the strong identifiability condition is a polynomial of an affine function. For simplicity, let us consider $h(x,\eta)=a^{\top}x+b$, where $\eta=(a,b)$.  Then, the function $F$ mentioned in the challenge (C.1) becomes $F(x;\beta_1,a,b)=\exp(\beta_{1}^{\top}x)(a^{\top}x+b)$. Under this seemingly unproblematic settings, we encounter an unexpected  phenomenon. Specifically, there exists an interaction between the gating parameter $\beta_1$ and the expert parameters $a,b$, captured by the partial differential equation (PDE)
\begin{align}
    \label{eq:PDE}
    \frac{\partial^2 F}{\partial\beta_{1}\partial b}(x;\beta^*_{1i},a^*_{i},b^*_{i})=\frac{\partial F}{\partial a}(x;\beta^*_{1i},a^*_{i},b^*_{i}).
\end{align}
Complex functional interactions of this form are not new -- they have been thoroughly characterized in the softmax gating Gaussian MoE model by \cite{nguyen2023demystifying}. However, 
and contrary to the case of data drawn from a well-specified softmax gating Gaussian MoE model, in our setting the above interaction causes the estimation rate of all the parameters $\beta^*_{1i},a^*_{i},b^*_{i}$ to be slower than any polynomial rates, and thus, could potentially be $\mathcal{O}_{P}(1/\log(n))$. It is important to note that this singular, rather surprising,  phenomenon takes place as we consider a deterministic MoE model instead of a probabilistic one, which requires us to develop new techniques.

\begin{table*}[!ht]
\caption{Summary of expert estimation rates (up to a logarithmic factor) under the softmax gating mixture of experts. In this work, we analyze three types of expert functions including strongly identifiable experts $h(x,\eta)$, ridge experts $\sigma(a^{\top}x+b)$ and polynomial experts $(a^{\top}x+b)^{p}$. For ridge experts, we consider two complement regimes: all the experts are input-dependent (Regime 1) vs. there exists an input-independent expert (Regime 2).  
Additionally, the notation $\mathcal{A}_j(\widehat{G}_n)$ stands for the Voronoi cels defined in equation~\eqref{eq:Voronoi_cells}.}
\textbf{}\\
\centering
\begin{tabular}{ | c | m{5em} | c | c | m{5em}|} 
\hline
\multirow{2}{4em}{\textbf{Expert Index}}& \multirow{2}{4em}{\textbf{Strongly-Identifiable Experts}} &\multicolumn{2}{c|}{\textbf{Ridge Experts with Strongly Independent Activation}} & \multirow{2}{4em}{\textbf{Polynomial Experts}}\\ \cline{3-4}
 & & Regime 1 & Regime 2 & \\
\hline
{$j:|\mathcal{A}_j(\widehat{G}_n)|=1$}  & \multicolumn{2}{c|}{${\mathcal{O}_{P}}(n^{-1/2})$} & \multicolumn{2}{c|}{Slower than ${\mathcal{O}_{P}}(n^{-1/2r}), \forall r\geq 1$ } \\
\hline
{$j:|\mathcal{A}_j(\widehat{G}_n)|>1$} & \multicolumn{2}{c|}{${\mathcal{O}_{P}}(n^{-1/4})$} & \multicolumn{2}{c|}{Slower than ${\mathcal{O}_{P}}(n^{-1/2r}), \forall r\geq 1$ }  \\
\hline
\end{tabular}
\label{table:expert_rates}
\end{table*}

\textbf{Overall contributions.} 
Our contributions are three-fold and can be summarized as follows (see also Table~\ref{table:expert_rates} for a summary of the expert estimation rates):

\textbf{1. Parametric rate for regression function.} In our first main result, Theorem~\ref{theorem:l2_bound}, we demonstrate a parametric estimation rate for the regression function $f_{G_*}(\cdot)$. In particular, we show that $\normf{f_{\widehat{G}_n}-f_{G_*}}=\mathcal{O}_P(n^{-1/2})$, where $\normf{\cdot}$ denotes the $L^2$ norm with respect to the probability measure $\mu$ of the input $X$. This result will be leveraged to obtain more complex estimation rates for the model parameters.  

\textbf{2. Strongly identifiable experts.} We formulate a general strong identifiability condition for expert functions in Definition~\ref{def:general_conditions} which ensures a faster, even parametric, estimation rates for the model parameters. 
To that effect, we propose a novel loss function $\mathcal{D}_{1}$ among parameters in equation~\eqref{eq:D1_loss} and establish in Theorem~\ref{theorem:general_experts} the $L^2$-lower bound $\normf{f_{G}-f_{G_*}}\gtrsim\mathcal{D}_{1}(G,G_*)$ for any $G\in\mathcal{G}_{k}(\Theta)$. Given the bound $\normf{f_{\widehat{G}_n}-f_{G_*}}=\mathcal{O}_{P}(n^{-1/2})$ in Theorem~\ref{theorem:l2_bound}, we deduce that the convergence rate of the LSE $\widehat{G}_n$ to the true mixing measure $G_*$ is also parametric on the sample size, i.e. $\mathcal{D}_{1}(\widehat{G}_n,G_*)=\mathcal{O}_{P}(n^{-1/2})$. This leads to an expert estimation rate of order at least $\mathcal{O}_{P}(n^{-1/4})$.

\textbf{3. Ridge experts:} Secondly, we focus ridge expert functions consisting of simple two-layer neural networks, which include a linear layer followed by an activation layer, i.e., $h(x,\eta)=\sigma(a^{\top}x+b)$, where $\eta=(a,b)$. In these very common settings, we give a condition called \emph{strong independence} in Definition~\ref{def:activation_conditions} to characterize activation functions that induce faster expert estimation rates. Interestingly, under the strongly independent settings of the activation function $\sigma$, we demonstrate in Theorem~\ref{theorem:singular_activation_experts} that even when the activation function $\sigma$ is strongly independent, the expert estimation rates are still slower than any polynomial rates and could be as slow as $\mathcal{O}_{P}(1/\log(n))$ if at lease one among parameters $a^*_{1},\ldots,a^*_{k_*}$ vanishes. Otherwise, we show in Theorem~\ref{theorem:activation_experts} that the expert estimation rates are no worse than $\mathcal{O}_{P}(n^{-1/4})$.

Lastly, we consider the settings when the activation function $\sigma$ is not strongly independent, e.g., polynomial experts of the form $h(x,\eta)=(a^{\top}x+b)^{p}$, where $p\in\mathbb{N}$ and $\eta=(a,b)$ (of which linear experts are special cases). This choice can be regarded as an ridge expert associated with the activation function $\sigma(z)=z^{p}$, which violates the strong independence condition. As a consequence, we come across an unforeseen phenomenon in Theorem~\ref{theorem:linear_experts}: the rates for estimating experts become universally worse than any polynomial rates due to an intrinsic interaction between gating and expert parameters via the PDE~\eqref{eq:PDE}. 

\textbf{Practical implications.} There are two main practical implications from our theoretical results: 

\textbf{(i) Expert network design.} Firstly, based on the the strong identifiability condition provide in Definition~\ref{def:general_conditions}, we can verify that plenty of widely used expert functions, namely feed-forward networks with activation functions $\sigmoid(\cdot)$, $\tanh(\cdot)$ and $\gelu(\cdot)$, are strongly identifiable. Therefore, our findings suggest that such experts enjoy faster estimation rates than others. This indicates that our theory is potentially useful for designing experts in practical applications. 

\textbf{(ii) Sample inefficiency of polynomial experts.} Secondly, Theorem~\ref{theorem:linear_experts} reveals that a class of polynomial experts, including linear experts, are not good choices of expert functions for MoE models due to its significantly slow estimation rates. This observation aligns with the findings in \cite{chen2022theory} which claims that a mixture of non-linear experts achieves a way better performance than a mixture of linear experts.

\textbf{Outline.} The paper is organized as follows. In Section~\ref{sec:background}, we obtain a parametric rate for the least squares estimation of softmax gating MoE model $f_{G_*}(\cdot)$ under the $L^2$-norm. Subsequently, we establish  estimation rates for experts that satisfy the strong identifiability condition in Section~\ref{sec:general_experts}. We then investigate ridge experts, including polynomial experts in Section~\ref{sec:activation_experts}. Finally, we conclude the paper and provide some future directions in Section~\ref{sec:conclusion}. Rigorous proofs and a simulation study are deferred to the supplementary material.

\textbf{Notations.} We let $[n]$ stand for the set $\{1,2,\ldots,n\}$ for any  $n\in\mathbb{N}$. Next, for any set $S$, we denote $|S|$ as its cardinality. For any vectors $v:=(v_1,v_2,\ldots,v_d) \in \mathbb{R}^{d}$ and $\alpha:=(\alpha_1,\alpha_2,\ldots,\alpha_d)\in\mathbb{N}^d$, we let $v^{\alpha}=v_{1}^{\alpha_{1}}v_{2}^{\alpha_{2}}\ldots v_{d}^{\alpha_{d}}$, $|v|:=v_1+v_2+\ldots+v_d$ and $\alpha!:=\alpha_{1}!\alpha_{2}!\ldots \alpha_{d}!$, while $\|v\|$ denotes its $2$-norm value. Lastly, for any two positive sequences $\{a_n\}_{n\geq 1}$ and $\{b_n\}_{n\geq 1}$, we write $a_n = \mathcal{O}(b_n)$ or $a_{n} \lesssim b_{n}$ if $a_n \leq C b_n$ for all $ n\in\mathbb{N}$, where $C > 0$ is some universal constant. The notation $a_{n} = \mathcal{O}_{P}(b_{n})$ indicates that $a_{n}/b_{n}$ is stochastically bounded. 

\section{The Estimation Rate for the Regression Function}
\label{sec:background}

In this section, we establish an important result, showing that, under minimal assumptions on the regression function, the least squares plug-in estimator of the regression function $f_{\widehat{G}_n}(\cdot)$ is consistent, and converges to the true regression function $f_{G_*}(\cdot)$ at the rate $1/\sqrt{n}$ with respect to the $L^2(\mu)$-distance, where $\mu$ is the feature distribution.

\textbf{Assumptions.} Throughout the paper, we impose the following standard assumptions on the model parameters. We recall that the dimension of the parameter space is fixed.

\textbf{(A.1)} We assume that the parameter space $\Theta$ is a compact subset of $\mathbb{R}\times\mathbb{R}^d\times\mathbb{R}^q$, while the input space $\mathcal{X}$ is bounded. These assumptions help guarantee the convergence of least squares estimation.

\textbf{(A.2)} For the experts $h(x,\eta^*_{1}),\ldots,h(x,\eta^*_{k_*})$ being different from each other, we assume that parameters $\eta^*_{1},\ldots,\eta^*_{k_*}$ are pair-wise distinct. Furthermore, these experts functions are Lipschitz continuous with respect to their parameters and bounded.

\textbf{(A.3)} In order that the softmax gating MoE $f_{G_*}(\cdot)$ is identifiable, i.e., $f_{G}(x)=f_{G_*}(x)$ for almost every $x$ implies that $G\equiv G_*$, we let $\beta^*_{0k_*}=0$ and $\beta^*_{1k_*}=\zerod$.

\textbf{(A.4)} To ensure that the softmax gate is input-dependent, we assume that at least one among gating parameters $\beta^*_{11},\ldots,\beta^*_{1k_*}$ is non-zero.

\begin{theorem}
    \label{theorem:l2_bound}
    Given a least squares estimator $\widehat{G}_n$ defined in equation~\eqref{eq:least_squared_estimator}, the model estimation $f_{\widehat{G}_n}$ admits the following convergence rate:
    \begin{align}
        \label{eq:model_bound}
        \normf{f_{\widehat{G}_n}-f_{G_*}}=\mathcal{O}_{P}(\sqrt{\log(n)/n}).
    \end{align}
\end{theorem}

The proof of Theorem~\ref{theorem:l2_bound} is in Appendix~\ref{appendix:l2_bound}. It can be seen from the bound~\eqref{eq:model_bound} that the rate for estimating the entire softmax gating MoE model $f_{G_*}(\cdot)$ is of order $\mathcal{O}_{P}(n^{-1/2})$ (up to logarithmic factor), which is parametric on the sample size $n$. More importantly, this result suggests that if we can construct a loss function among parameters $\mathcal{D}$ such that $\normf{f_{\widehat{G}_n}-f_{G_*}}\gtrsim \mathcal{D}(\widehat{G}_n,G_*)$, then it follows that $\mathcal{D}(\widehat{G}_n,G_*)=\mathcal{O}_P(n^{-1/2})$. As a consequence, we achieve parameter estimation rates through the previous bound, and therefore, our desired expert estimation rates.

\section{Strongly Identifiable Experts}
\label{sec:general_experts}
In this section, we derive estimation rates for the parameters of the softmax gating MoE regression function~\eqref{eq:standard_MoE} assuming that the class of expert functions satisfy a novel regularity condition which we refer to as \emph{strong identifiability}; see  Definition~\ref{def:general_conditions} below. 

Let us recall that in order to establish the expert estimation rates, our approach is to establish the $L^2$-lower bound $\normf{f_{\widehat{G}_n}-f_{G_*}}\gtrsim \mathcal{D}(\widehat{G}_n,G_*)$ mentioned in Section~\ref{sec:background}, where $\mathcal{D}$ is an appropriate loss function to be defined later. For that purpose, a key step is to decompose the quantity $f_{\widehat{G}_n}(x)-f_{G_*}(x)$ into a combination of linearly independent terms, where
\begin{align*}
    f_{G_{*}}(x) := \sum_{i=1}^{k_*} \frac{\exp((\beta^*_{1i})^{\top}x+\beta^*_{0i})}{\sum_{j=1}^{k_*}\exp((\beta^*_{1j})^{\top}x+\beta^*_{0j})}\cdot h(x,\eta^*_i).
\end{align*}
This can be done by using Taylor expansions to the product of a softmax numerator and an expert denoted by $x\mapsto F(x;\beta_{1},\eta)=\exp(\beta_{1}^{\top}x)h(x,\eta)$. Therefore, to obtain our desired decomposition, we present in the following definition a condition that ensures the derivatives of $F$ with respect to its parameters are linearly independent.
\begin{definition}[Strong Identifiability]
    \label{def:general_conditions}
    We say that an expert function $x \mapsto h(x,\eta)$ is strongly identifiable if it is twice differentiable with respect to its parameter $\eta$ and the following set of functions in $x$ is linearly independent:
    \begin{align*}
        &\Big\{x^{\nu}\cdot\frac{\partial^{|\tau|}h}{\partial\eta^{\tau}}(x,\eta_j): j\in[k], \nu\in\mathbb{N}^d, \ \tau\in\mathbb{N}^{q}, \\
        &\hspace{5cm}0\leq |\nu|+|\tau|\leq 2\Big\},
    \end{align*}
    for almost every $x$ for any $k\geq 1$ and pair-wise distinct parameters $\eta_1,\ldots,\eta_k$.
\end{definition}
As indicated in Definition~\ref{def:general_conditions}, the main distinction between the strong identifiability and standard identifiability conditions of the expert function $h$ \cite{ho2022gaussian} is that we further require the first and second-order derivatives of the expert function $h$ with respect to their parameter are also linearly independent. Intuitively, the linear independence of functions in Definition~\ref{def:general_conditions} helps eliminate potential interactions among parameters expressed in the language of partial differential equations (see e.g., equation~\eqref{eq:PDE_efficient_activation} and equation~\eqref{eq:PDE_polynomial} where gating parameters $\beta_1$ interact with expert parameters $a$). Such interactions are demonstrated to result in significantly slow expert estimation rates (see Theorem 4.4 and Theorem 4.6). 

\textbf{Example.} It can be checked that the strong identifiability condition holds for several experts used in practice, including feed-forward networks with activations like $\mathrm{sigmoid}$, $\mathrm{tanh}$ and $\mathrm{GeLU}$ and non-linear transformed input are strongly identifiable experts. For simplicity, let us consider a 2-layer expert network with normalized input, i.e.
\begin{align*}
    h(x,(a,b))=\sigma\Big(a\frac{x}{\|x\|}+b\Big),
\end{align*}
where $\sigma$ is the sigmoid function and $x,a,b,\in\mathbb{R}$ (The argument also holds with layer normalization as well). Then, by taking the derivatives of the expert function $h(\cdot,(a,b))$ w.r.t its parameters up to the second order, we can verify that the set mentioned in Definition~\ref{def:general_conditions} is linearly independent, which means that the expert $h(\cdot,(a,b))$ is strongly identifiable. The non-linear transformation is involved to ensure the linearly independence between the terms $\frac{\partial h}{\partial a}$ and $x\frac{\partial h}{\partial b}$ mentioned in Definition 3.1. Otherwise, the strong identifiability condition is not satisfied. For instance, the ridge expert $h(x,(a,b))=\sigma(ax+b)$ is not strongly identifiable due to the PDE $\frac{\partial h}{\partial a}=x\frac{\partial h}{\partial b}$. 

Next, to compute the expert estimation rates, we propose a loss function based on the notion of Voronoi cells, put forward by \cite{manole22refined}, as follows.

\textbf{Voronoi loss.} Given an arbitrary mixing measure $G$ with $k'\leq k$ components, we partition its components to the following Voronoi cells $\mathcal{A}_j\equiv\mathcal{A}_j(G)$, which are generated by the components of $G_*$:
\begin{align}
    \label{eq:Voronoi_cells}
    \mathcal{A}_j:=\{i\in[k']:\|\omega_i-\omega^*_j\|\leq\|\omega_i-\omega^*_{\ell}\|,\forall \ell\neq j\},
\end{align}
where $\omega_i:=(\beta_{1i},\eta_i)$ and $\omega^*_j:=(\beta^*_{1j},\eta^*_j)$ for any $j\in[k_*]$. Notably, the cardinality of Voronoi cell $\mathcal{A}_j$ is exactly the number of fitted components that approximates $\omega^*_j$. Then, the Voronoi loss function used for our analysis is given by:
\begin{align}
\label{eq:D1_loss}
&\mathcal{D}_1(G,G_*):=\sum_{j=1}^{k_*}\Big|\sum_{i\in\mathcal{A}_j}\exp(\beta_{0i})-\exp(\beta^*_{0j})\Big|\nonumber\\
    &+\sum_{j:|\mathcal{A}_j|>1}\sum_{i\in\mathcal{A}_j}\exp(\beta_{0i})\Big[\|\Delta\beta_{1ij}\|^2+\|\Delta \eta_{ij}\|^2\Big]\nonumber\\
    &+\sum_{j:|\mathcal{A}_j|=1}\sum_{i\in\mathcal{A}_j}\exp(\beta_{0i})\Big[\|\Delta\beta_{1ij}\|+\|\Delta \eta_{ij}\|\Big],
\end{align}
where we denote $\Delta\beta_{1ij}:=\beta_{1i}-\beta^*_{1j}$ and $\Delta \eta_{ij}:=\eta_i-\eta^*_j$. Above, if the Voronoi cell $\mathcal{A}_j$ is empty, then we let the corresponding summation term be zero. Additionally, it can be checked that $\mathcal{D}_{1}(G,G_*)=0$ if and only if $G\equiv G_*$. Thus, when $\mathcal{D}_{1}(G,G_*)$ is sufficiently small, the differences $\Delta\beta_{1ij}$ and $\Delta\eta_{ij}$ are also small. This property indicates that $\mathcal{D}_{1}(G,G_*)$ is an appropriate loss function for measuring the discrepancy between the LSE $\widehat{G}_n$ and the true mixing measures $G_*$. However, since the loss $\mathcal{D}_1(G,G_*)$ is not symmetric, it is not a proper metric. Finally, computing the Voronoi loss function $\mathcal{D}_{1}$ is efficient as its computational complexity is at the order of $\mathcal{O}(k \times k_{*})$.

Equipped with the Voronoi loss function $\mathcal{D}_{1}$, we are now ready to characterize the parameter estimation rates as well as the expert estimation rates in the following theorem.

\begin{theorem}
    \label{theorem:general_experts}
    Suppose that the expert function $h(x,\eta)$ satisfies the condition in Definition~\ref{def:general_conditions}, then the following $L^2$-lower bound holds true for any $G\in\mathcal{G}_{k}(\Theta)$:
    \begin{align*}
        \normf{f_{G}-f_{G_*}}\gtrsim\mathcal{D}_1(G,G_*).
    \end{align*}
    Furthermore, this bound and the result in Theorem~\ref{theorem:l2_bound} imply that $\mathcal{D}_1(\widehat{G}_n,G_*)=\mathcal{O}_{P}(\sqrt{\log(n)/n})$.
\end{theorem}
The proof of Theorem~\ref{theorem:general_experts} is in Appendix~\ref{appendix:general_experts}. A few remarks regarding the results of Theorem~\ref{theorem:general_experts} are in order.

\textbf{(i)} Firstly, the parameters $\beta^*_{1j},\eta^*_{1j}$ that are approximated by more than one component, i.e. those for which $|\mathcal{A}_{j}(\widehat{G}_n)|>1$, enjoy the same estimation rate of order $\mathcal{O}_{P}(n^{-1/4})$. Additionally, since the expert $h(x,\eta)$ is twice differentiable over a bounded domain, it is also a Lipschitz function. Therefore, by denoting $\widehat{G}_n:=\sum_{i=1}^{\widehat{k}_n}\exp(\widehat{\beta}_{0i})\delta_{(\widehat{\beta}^n_{1i},\widehat{\eta}^n_i)}$, we obtain that
\begin{align}
    \sup_{x} |h(x,\widehat{\eta}^n_i)-h(x,\eta^*_j)| &\leq L_1~\|\widehat{\eta}^n_i-\eta^*_j\| \nonumber \\
    &\lesssim \mathcal{O}_{P}(n^{-1/4}),     \label{eq:expert_rate}
\end{align}
for any $i\in\mathcal{A}_j(\widehat{G}_n)$, where $L_1\geq 0$ is a Lipschitz constant. Consequently, the rate for estimating a strongly identifiable expert $h(x,\eta^*_j)$ continues to be $\mathcal{O}_{P}(n^{-1/4})$ as long as it is fitted by more than one expert. On the other hand, when considering the softmax gating Gaussian MoE, \cite{nguyen2023demystifying} pointed out that the estimation rates for linear experts could be $\mathcal{O}_{P}(n^{-1/12})$ when they are fitted by three experts, i.e., $|\mathcal{A}_j(\widehat{G}_n)|=3$. Moreover, these rates will become even slower if their number of fitted experts increases. This comment highlights how the strong identifiability condition proposed in this paper immediately implies fast estimation rates, 

\textbf{(ii)} Secondly, the rates for estimating parameters $\beta^*_{1j},\eta^*_{j}$ that are fitted by exactly one component, i.e., $|\mathcal{A}_j(\widehat{G}_n)|=1$, are faster than those in Remark (i), of order $\mathcal{O}_{P}(n^{-1/2})$. By employing the same arguments as in equation~\eqref{eq:expert_rate}, we deduce that the expert $h(x,\eta^*_j)$ admits the estimation rate of order $\mathcal{O}_{P}(n^{-1/2})$, which matches its counterpart in \cite{nguyen2023demystifying}.

\section{Ridge Experts}
\label{sec:activation_experts}
In this section, we turn to the softmax gating MoE models with ridge experts, i.e  two-layer neural networks comprised of a linear layer and an activation layer of the form
\begin{align}
    \label{eq:activation_MoE}
    h(x,\eta^*_j)=\sigma((a^*_j)^{\top}x+b^*_j),
\end{align}
where $\sigma:\mathbb{R}\to\mathbb{R}$ is the (usually, nonlinear) activation function and $\eta^*_{j}=(a^*_j,b^*_j)\in\mathbb{R}^d\times\mathbb{R}$ are expert parameters. Ridge experts are commonly deployed in deep-learning architectures and generative models, and they fail to satisfy the strong identifiability condition from the last section. To overcome this issue, we instead formulate a \emph{strong independence} condition on the activation function itself, which will guarantee fast estimation rates, provided that all the expert parameters are non-zero. 
Interestingly, when one or more parameters are zero, so that the corresponding experts are constant functions, we show slow, non-polynomial rates in the sample size.


In Section~\ref{sec:polynomial_experts}, we then examine polynomial activation functions, which violates the strong independence condition. In this case we again demonstrate slow rates.

\subsection{On Strongly Independent Activation}
\label{sec:efficient_activation}
To begin with, let us recall from Section~\ref{sec:general_experts} that our goal is to establish the $L^2$-lower bound $\normf{f_{\widehat{G}_n}-f_{G_*}}\gtrsim \mathcal{D}(\widehat{G}_n,G_*)$ where $G_*:=\sum_{i=1}^{k_*}\exp(\beta^*_{0i})\delta_{(\beta^*_{1i},a^*_{i},b^*_{i})}$,
\begin{align*}
     f_{G_{*}}(x) &:= \sum_{i=1}^{k_*} \frac{\exp((\beta^*_{1i})^{\top}x+\beta^*_{0i})}{\sum_{j=1}^{k_*}\exp((\beta^*_{1j})^{\top}x+\beta^*_{0j})}\\
     &\hspace{4cm}\times\sigma((a^*_i)^{\top}x+b^*_i),
\end{align*}
and $\mathcal{D}$ is a loss function among parameters that will be defined later. In our proof techniques, we first need to represent the term $f_{\widehat{G}_n}(x)-f_{G_*}(x)$ as a weighted sum of linearly independent terms by applying Taylor expansions to the function $F(X;\beta_{1},a,b)=\exp(\beta_{1}^{\top}x)\sigma(a^{\top}x+b)$. Nevertheless, we notice that if $a^*_i=\zerod$ for some $i\in[k_*]$, then there is an interaction between gating and expert parameters expressed in the language of PDE as follows:
\begin{align}
    \label{eq:PDE_efficient_activation}
    \frac{\partial F}{\partial \beta_1}(x;\beta^*_{1i},a^*_{i},b^*_{i})=\sigma'(b^*_i)\cdot\frac{\partial F}{\partial a}(x;\beta^*_{1i},a^*_{i},b^*_{i}).
\end{align}
The above PDE leads to a number of linearly dependent terms in the decomposition of $f_{\widehat{G}_n}(x)-f_{G_*}(x)$, which could negatively affect the expert estimation rates. To understand the effects of the previous interaction better, we split the analysis into two following regimes of parameters $a^*_i$ where the interaction~\eqref{eq:PDE_efficient_activation} vanishes and occurs, respectively:
\begin{itemize}
    \item \textbf{Regime 1:} All parameters $a^*_1,\ldots,a^*_{k_*}$ are different from $\zerod$;
    \item \textbf{Regime 2:} At least one among parameters $a^*_1,\ldots,a^*_{k_*}$ is equal to $\zerod$.
\end{itemize}
Subsequently, we will conduct an expert convergence analysis in each of the above regimes.

\subsubsection{Regime 1: Input-dependent Experts}
\label{sec:input-dependent_experts}
Under this regime, since all the parameters $a^*_{1},\ldots,a^*_{k_*}$ are different from $\zerod$, the PDE~\eqref{eq:PDE_efficient_activation} does not hold true, and thus, we do not need to deal with linearly dependent terms induced by this PDE. Instead, we establish a strong independence condition on the activation $\sigma$ in Definition~\ref{def:activation_conditions} to guarantee that there are no interactions among parameters, i.e. the derivatives of the function $x\mapsto F(x;\beta_{1},a,b)=\exp(\beta_{1}^{\top}x)\sigma(a^{\top}x+b)$ and its derivatives up to the second order are linearly independent.

\begin{definition}[Strong Independence]
    \label{def:activation_conditions}
    We say that an activation function $\sigma:\mathbb{R}\to\mathbb{R}$ is \emph{strongly independent} if it is twice differentiable and the  set of functions in $x$
    \begin{align*}
        \Big\{x^{\nu}\sigma^{(\tau)}(a_j^{\top}x+b_j):&\nu\in\mathbb{N}^d,\tau\in\mathbb{N}, \\
        &\ 0\leq |\nu|,\tau\leq 2, \ j\in[k]\Big\},
    \end{align*}
    is linearly independent, for almost all $x$,
    for any pair-wise distinct parameters $(a_1,b_1),\ldots,(a_k,b_k)$ and $k\geq 1$, where $\sigma^{(\tau)}$ denotes the $\tau$-th derivative of $\sigma$. 
\end{definition}
\textbf{Example.} We can verify that $\sigmoid(\cdot)$ and Gaussian error linear units $\gelu(\cdot)$ \cite{hendrycks2023gaussian} are strongly independent activation functions. By contrast, the polynomial activation $\sigma(z)=z^{p}$ is not strongly independent for any $p\geq 1$.

Just like in the previous section, we construct an appropriate loss function among parameters that is upper bounded by the $L^2(\mu)$ distance between the corresponding softmax gating MoE regression functions to obtain expert estimation rates. 

\textbf{Voronoi loss.} Tailored to the setting of Regime 1, the Voronoi loss of interest is given by
\begin{align*}
    &\mathcal{D}_2(G,G_*):=\sum_{j=1}^{k_*}\Big|\sum_{i\in\mathcal{A}_j}\exp(\beta_{0i})-\exp(\beta^*_{0j})\Big|\nonumber\\
    &+\sum_{j:|\mathcal{A}_j|>1}\sum_{i\in\mathcal{A}_j}\exp(\beta_{0i})\Big[\|\Delta\beta_{1ij}\|^2+\|\Delta a_{ij}\|^2+|\Delta b_{ij}|^2\Big]\nonumber\\
    &+\sum_{j:|\mathcal{A}_j|=1}\sum_{i\in\mathcal{A}_j}\exp(\beta_{0i})\Big[\|\Delta\beta_{1ij}\|+\|\Delta a_{ij}\|+|\Delta b_{ij}|\Big],
\end{align*}
where we denote $\Delta a_{ij}:=a_i-a^*_j$ and $\Delta b_{ij}:=b_i-b^*_j$.

\begin{theorem}
    \label{theorem:activation_experts}
    Assume that the experts take the form $\sigma(a^{\top}x+b)$, where the activation function $\sigma(\cdot)$ satisfies the condition in Definition~\ref{def:activation_conditions}, then the following $L^2$-lower bound holds true for any $G\in\mathcal{G}_{k}(\Theta)$ under the Regime 1:
    \begin{align*}
        \normf{f_{G}-f_{G_*}}\gtrsim\mathcal{D}_2(G,G_*).
    \end{align*}
    Furthermore, this bound and the result in Theorem~\ref{theorem:l2_bound} imply that $\mathcal{D}_{2}(\widehat{G}_n,G_*)=\mathcal{O}_{P}(\sqrt{\log(n)/n})$. 
\end{theorem}
See Appendix~\ref{appendix:activation_experts} for a proof. Theorem~\ref{theorem:activation_experts} indicates the LSE $\widehat{G}_n$ converges to $G_*$ at the parametric rate $\mathcal{O}_{P}(n^{-1/2})$ under the loss function $\mathcal{D}_2$. From the formulation of this loss, we deduce the following rates. 

\textbf{(i)} For parameters $\beta^*_{1j}$, $a^*_{j}$ and $b^*_{j}$  fitted by one component, i.e., $|\mathcal{A}_j(\widehat{G}_n)|=1$, the estimation rate is of order $\mathcal{O}_P(n^{-1/2})$. Moreover, as the strongly independent $\sigma$ is twice differentiable, the function $x\mapsto\sigma(a^{\top}x+b)$ is Lipschitz continuous with some Lipschitz constant $L_2\geq 0$. Thus, denoting $\widehat{G}_n:=\sum_{i=1}^{\widehat{k}_n}\exp(\widehat{\beta}_{0i})\delta_{(\widehat{\beta}^n_{1i},\widehat{a}^n_i,\widehat{b}^n_i)}$, we have
\begin{align}
    \label{eq:activation_rate}
    & \sup_{x} |\sigma((\widehat{a}^n_i)^{\top}x+\widehat{b}^n_i)-\sigma((a^*_j)^{\top}x+b^*_j)|\nonumber\\
    &~\leq L_2\cdot\|(\widehat{a}^n_i,\widehat{b}^n_i)-(a^*_j,b^*_j)\|\nonumber\\
    &~\leq L_2\cdot(\|\widehat{a}^n_i-a^*_j\|+|\widehat{b}^n_i-b^*_j|)\nonumber\\
    &~\lesssim\mathcal{O}_P(n^{-1/2}).
\end{align}
As a consequence, the estimation rate for the expert $\sigma((a^*_j)^{\top}x+b^*_j)$ is also of order $\mathcal{O}_P(n^{-1/2})$.

\textbf{(ii)} For parameters, say $\beta^*_{1j}$, $a^*_{j}$ and $b^*_{j}$, fitted by more than one component, i.e. $|\mathcal{A}_j(\widehat{G}_n)|>1$, the corresponding rates are $\mathcal{O}_{P}(n^{-1/4})$. By reusing the arguments in equation~\eqref{eq:activation_rate}, we deduce that the expert $\sigma((a^*_j)^{\top}x+b^*_j)$ admits the estimation rate of order $\mathcal{O}_P(n^{-1/4})$.


\subsubsection{Regime 2: Input-independent Experts}
Recall that under this regime, at least one among parameters $a^*_1,\ldots,a^*_{k_*}$ equal to $\zerod$. Without loss of generality, we may assume that $a^*_1=\zerod$. This means that the the value of the first expert $\sigma((a^*_1)^{\top}x+b^*_1)$ no longer depends on the input $x$. In this case, there exists an interaction among the gating parameter $\beta_{1}$ and the expert parameter $a$ captured by the PDE
\begin{align}
    \label{eq:regime2_PDE}
    \frac{\partial F}{\partial \beta_1}(x;\beta^*_{11},a^*_{1},b^*_{1})=\sigma'(b^*_1)\cdot\frac{\partial F}{\partial a}(x;\beta^*_{11},a^*_{1},b^*_{1}).
\end{align}
The significance of this fact is that, owing to the the above PDE, the following Voronoi loss function among parameters is not majorized by the $L^2(\mu)$ distance between the corresponding expert functions:
\begin{align}
    \label{eq:D_3}
    &\mathcal{D}_{3,r}(G,G_*):=\sum_{j=1}^{k_*}\Big|\sum_{i\in\mathcal{A}_j}\exp(\beta_{0i})-\exp(\beta^*_{0j})\Big|\nonumber\\
    &+\sum_{j=1}^{k_*}\sum_{i\in\mathcal{A}_j}\exp(\beta_{0i})\Big[\|\Delta\beta_{1ij}\|^r+\|\Delta a_{ij}\|^r+|\Delta b_{ij}|^r\Big],
\end{align}
for any $r\geq 1 $. This is formalized in the next result, whose proof can be found in Appendix~\ref{appendix:activation_limit}.
\begin{proposition}
    \label{prop:activation_limit}
    Let the expert function take the form $\sigma(a^{\top}x+b)$, and suppose that  not all the parameters $a^*_{1},\ldots,a^*_{k_*}$ are different from $\zerod$, then we obtain that
    \begin{align*}
        \lim_{\varepsilon\to0}\inf_{\substack{G\in\mathcal{G}_{k}(\Theta): \\ \mathcal{D}_{3,r}(G,G_*)\leq\varepsilon}}\normf{f_{G}-f_{G_*}}/\mathcal{D}_{3,r}(G,G_*)=0,
    \end{align*}
    for any $r\geq 1$.
\end{proposition}
The above proposition, combined with Theorem~\ref{theorem:l2_bound}, indicates that the parameter estimation rate in this situation ought to be slower than any polynomial of $1/\sqrt{n}$. 
This is indeed the case, as confirmed by the following minimax lower bound. 
\begin{theorem}
    \label{theorem:singular_activation_experts}
    Assume that the experts take the form $\sigma(a^{\top}x+b)$, then the following minimax lower bound of estimating $G_*$ holds true for any $r\geq 1$ under the Regime 2:
    \begin{align*}
        \inf_{\overline{G}_n\in\mathcal{G}_{k}(\Theta)}\sup_{G\in\mathcal{G}_{k}(\Theta)\setminus\mathcal{G}_{k_*-1}(\Theta)}\bbE_{f_{G}}[\mathcal{D}_{3,r}(\overline{G}_n,G)]\gtrsim n^{-1/2},
    \end{align*}
    where $\bbE_{f_{G}}$ indicates the expectation taken w.r.t the product measure with $f^n_{G}$ and the infimum is over all estimators taking values in $\mathcal{G}_{k}$.
\end{theorem}
The proof of Theorem~\ref{theorem:singular_activation_experts} is in Appendix~\ref{appendix:singular_activation_experts}. This result together with the formulation of the Voronoi loss $\mathcal{D}_{3,r}$ in equation~\eqref{eq:D_3} leads to a singular and striking phenomenon that, to the best of our knowledge, has never been observed in previous work \cite{ho2022gaussian,chen2022theory,nguyen2023demystifying}. Specifically, 

\textbf{(i)} The rates for estimating the parameters $\beta^*_{1j}$, $a^*_{j}$ and $b^*_{j}$ are slower than any polynomial rate $\mathcal{O}_{P}(n^{-1/2r})$ for any $r\geq 1$. In particular, they could be as slow as $\mathcal{O}_{P}(1/\log(n))$. 

\textbf{(ii)} Recall from equation~\eqref{eq:activation_rate} that 
\begin{align}
    \label{eq:activation_bound}
    &\sup_{x} |\sigma((\widehat{a}^n_i)^{\top}x+\widehat{b}^n_i)-\sigma((a^*_j)^{\top}x+b^*_j)|\nonumber\\
    &\hspace{2cm}\leq L_2\cdot(\|\widehat{a}^n_i-a^*_j\|+|\widehat{b}^n_i-b^*_j|).
\end{align}
Consequently, the expert estimation rates might also be significantly slow, of order $\mathcal{O}_{P}(1/\log(n))$ or worse, due to the interaction between gating and expert parameters in equation~\eqref{eq:regime2_PDE}. It is worth noting that these slow rates occur even when the activation function $\sigma$ meets the strong independence condition in Definition~\ref{def:activation_conditions}. This observation suggests that all the expert parameters $a^*_{1},\ldots,a^*_{k_*}$ should be different from $\zerod$. In other words, every expert of the form $\sigma(a^{\top}x+b)$ in the MoE model should depend on the input value. 

\subsection{On Polynomial Activation}
\label{sec:polynomial_experts}
We now focus on a specific setting in which the activation function $\sigma$ is formulated as a polynomial, i.e. $\sigma(z)=z^{p}$, for some $p \in \mathbb{N}$. Concretely,  for all $j \in [k^*]$, we set
\begin{align}
    h(x,\eta^*_{j})=((a^*_{j})^{\top}x+b^*_{j})^{p}, \quad x \in \mathcal{X},
\end{align}
and call it a polynomial expert. Notably, it can be verified that this activation function violates the strong independence condition in Definition~\ref{def:activation_conditions} for any $p\in\mathbb{N}$. For simplicity, let us consider only the setting when $p=1$, i.e., $h(x,\eta^*_{j})=(a^*_{j})^{\top}x+b^*_{j}$, with a note that the results for other settings of $p$ can be argued in a similar fashion.

Since the strong independence condition in Definition~\ref{def:activation_conditions} is not satisfied, we have to deal with an interaction among parameters, capture by following PDE:
\begin{align}
    \label{eq:PDE_polynomial}
    \frac{\partial^2 F}{\partial\beta_{1}\partial b}(x;\beta^*_{1i},a^*_{i},b^*_{i})=\frac{\partial F}{\partial a}(x;\beta^*_{1i},a^*_{i},b^*_{i}),
\end{align}
\begin{table*}[!ht]
\caption{Comparison of parameter and expert estimation rates under the probabilistic softmax gating mixture of \textbf{linear experts} \cite{nguyen2023demystifying} and the deterministic one (Ours). Here, we denote $\bar{r}_j:=\bar{r}(|\mathcal{A}^n_j|)$, where the function $\bar{r}(\cdot)$ represents for the solvability of a system of polynomial equations in \cite{nguyen2023demystifying}. Some specific values of this function are given by: $\bar{r}(2)=4$ and $\bar{r}(3)=6$.}
\textbf{}\\
\centering
\begin{tabular}{ | c | c | c |c|c|c|c|} 
\hline
\multirow{2}{4em}{\textbf{}} &  \multicolumn{2}{c|}{\textbf{Parameters} $a^*_j$} &  \multicolumn{2}{c|}{\textbf{Parameters} $b^*_j$} & \multicolumn{2}{c|}{\textbf{Experts} $(a^*_j)^{\top}x+b^*_j$} \\ \cline{2-7}
\textbf{Model Type}& $j:|\mathcal{A}^n_j|=1$ & $j:|\mathcal{A}^n_j|>1$& $j:|\mathcal{A}^n_j|=1$ & $j:|\mathcal{A}^n_j|>1$ & $j:|\mathcal{A}^n_j|=1$ & $j:|\mathcal{A}^n_j|>1$\\
\hline 
Probabilistic & {$\mathcal{O}_{P}(n^{-1/2})$}& $\mathcal{O}_{P}(n^{-1/\bar{r}_j})$ & {$\mathcal{O}_{P}(n^{-1/2})$}& $\mathcal{O}_{P}(n^{-1/2\bar{r}_j})$ & {$\mathcal{O}_{P}(n^{-1/2})$}& $\mathcal{O}_{P}(n^{-1/2\bar{r}_j})$\\
\hline
Deterministic  & \multicolumn{6}{c|}{Slower than $\mathcal{O}_{P}(n^{-1/2r}), \forall r\geq 1$}\\
\hline
\end{tabular}

\label{table:comparison}
\end{table*}

\noindent
where $F(x;\beta_{1},a,b):=\exp(\beta_{1}^{\top}x)(a^{\top}x+b)$ is the product of softmax numerator and the expert function. Though this interaction has already been observed and analyzed in previous work \cite{nguyen2023demystifying}, its effects on the expert convergence  rate in the present settings are totally different as we consider a deterministic MoE model rather than a probabilistic model. In particular, \cite{nguyen2023demystifying} argued that the interaction~\eqref{eq:PDE_polynomial} led to polynomial expert estimation rates which were determined by the solvability of a system of polynomial equations. On the other hand, we show in the Proposition~\ref{prop:linear_limit} below that such interaction makes the ratio $\normf{f_{G}-f_{G_*}}/\mathcal{D}_{3,r}(G,G_*)$ vanish when the loss $\mathcal{D}_{3,r}(G,G_*)$ goes to zero as shown in the next proposition, whose proof can be found in Appendix~\ref{appendix:linear_limit}.
\begin{proposition}
    \label{prop:linear_limit}
    Let the expert functions take the form $a^{\top}x+b$, then the following limit holds for any $r\geq 1$:
    \begin{align*}
        \lim_{\varepsilon\to0}\inf_{\substack{G\in\mathcal{G}_{k}(\Theta): \\ \mathcal{D}_{3,r}(G,G_*)\leq\varepsilon}}\normf{f_{G}-f_{G_*}}/\mathcal{D}_{3,r}(G,G_*)=0.
    \end{align*}
\end{proposition}
Just like in the previous section, we arrive at a significantly slow expert estimation rates. 

\begin{theorem}
    \label{theorem:linear_experts}
    Assume that the experts take the form $a^{\top}x+b$, then we achieve the following minimax lower bound of estimating $G_*$:
    \begin{align*}
        \inf_{\overline{G}_n\in\mathcal{G}_{k}(\Theta)}\sup_{G\in\mathcal{G}_{k}(\Theta)\setminus\mathcal{G}_{k_*-1}(\Theta)}\bbE_{f_{G}}[\mathcal{D}_{3,r}(\overline{G}_n,G)]\gtrsim n^{-1/2},
    \end{align*}
    for any $r\geq 1$, where $\bbE_{f_{G}}$ indicates the expectation taken w.r.t the product measure with $f^n_{G}$.
\end{theorem}
Proof of Theorem~\ref{theorem:linear_experts} is in Appendix~\ref{appendix:linear_experts}. A few comments regarding the above theorem are in order (see also Table~\ref{table:comparison}):

\textbf{(i)} Theorem~\ref{theorem:linear_experts} reveals that using polynomial experts will result in the same slow rates as using input-independent experts, as described in Theorem~\ref{theorem:singular_activation_experts}. More specifically, the estimation rates for parameters $\beta^*_{1i}$, $a^*_{i}$ and $b^*_{i}$ are slower than any polynomial rates, and could be of order $\mathcal{O}_{P}(1/\log(n))$ because of the interaction in equation~\eqref{eq:PDE_polynomial}. 

\textbf{(ii)} Additionally, we have that
\begin{align*}
    & \sup_{x} \Big|((\widehat{a}^n_i)^{\top}x+\widehat{b}^n_i)-((a^*_j)^{\top}x+b^*_j)\Big|\nonumber\\
    &\hspace{2cm}\leq \sup_{x} \|\widehat{a}^n_i-a^*_j\|\cdot\|x\|+|\widehat{b}^n_i-b^*_j|.
\end{align*}
Since the input space $\mathcal{X}$ is bounded, we deduce that the rates for estimating polynomial experts $(a^*_j)^{\top}x+b^*_j$ could also be as slow as $\mathcal{O}_{P}(1/\log(n))$. This is remarkable, especially in contrast to the polynomial rates of linear expert  established by \cite{nguyen2023demystifying} in probabilistic softmax gating experts. Hence, for the expert estimation problem, the performance of a mixture of linear experts cannot compare to that of a mixture of non-linear experts. It is worth noting that this claim aligns with the findings in \cite{chen2022theory}.



\section{Conclusions}
\label{sec:conclusion}
In this paper, we have analyzed the convergence rates of the least squares estimator under a deterministic softmax gating MoE model. We have shown that expert functions that satisfy a novel condition  referred to as strong identifiability enjoy estimation rates of polynomial orders. When specializing to experts of the form ridge function, polynomial rates can be guaranteed under another condition, called strongly independent activation, provided that all the expert parameters are non-zero. In contrast, when at least one of the expert parameters vanishes, we have unveiled the surprising fact that expert estimation rates become slower than any polynomial rates. Furthermore, we also prove that polynomial experts, which violate the strong identifiability condition, also experience such slow rates under any parameter settings. 

\textbf{Limitations.} Our analysis has two following limitations:

\textbf{(L.1)} The theoretical results established in the paper are under the assumption that the data are generated from a regression model where the regression function is a softmax gating MoE. This assumption can be violated in real-world settings when
the data are not necessarily generated from that model. Under those misspecified settings, the regression function is an arbitrary function $g$ which is not necessarily a mixture of experts. Then, the least square estimator $\widehat{G}_n$ defined in equation~\eqref{eq:least_squared_estimator} converges to the mixing measure $\overline{G}\in\mathcal{G}_k$ that minimizes the $L^2$-distance between $f_{G}$ and $g$. Since the current analysis of the least square estimation under the misspecified settings of statistical models is mostly conducted when the function space is convex \cite{vandeGeer-00}, it is inapplicable to our setting where the space $\mathcal{G}_k$ is non-convex. Therefore, we believe that further techniques should be developed to analyze the misspecified settings, which is beyond the scope of our work.

\textbf{(L.2) } The depth of an expert network has not been considered in capturing the convergence rate of expert estimation. In particular, we demonstrate in Theorem~\ref{theorem:general_experts} that any choice of expert network which satisfies the strong identifiability condition will lead to polynomial expert estimation rates regardless of its depth. Secondly, although ridge experts of the form $h(x,(a,b))=\sigma(a^{\top}x+b)$ are not strongly identifiable, we show in Theorem~\ref{theorem:activation_experts} that if the activation $\sigma$ satisfies the strong independence condition in Definition~\ref{def:activation_conditions}, then the expert estimation rates are also polynomial. On the other hand, for ridge experts with activation $\sigma$ violating the strong independence condition, e.g. polynomial experts, we find that increasing the depth of the expert network would not help improve the slow expert estimation rates in Theorem~\ref{theorem:linear_experts} due to an intrinsic interaction among parameters of polynomial experts (expressed in the language of partial differential equations). 
We believe that technical tools need to be further developed to understand the effects of the network depth on the expert estimation problem. As it stays beyond the scope of our work, we leave that direction for future work. 

\textbf{Future directions.} There are some potential directions to which our current theory can extend. Firstly, we can leverage our techniques to capture the convergence behavior of different types of experts under the MoE models with other gating functions, namely Top-K sparse gate \cite{shazeer2017topk}, dense-to-sparse gate \cite{nie2022evomoe}, cosine similarity gate \cite{li2023sparse}, Laplace gate \cite{han2024fusemoe}, and sigmoid gate \cite{csordas2023approximating}. Such analysis would enrich the knowledge of expert selection given a specific gating function. Additionally, we can develop our current techniques to provide a comprehensive understanding of more complex MoE models such as hierarchical MoE \cite{zhao_hierarchical_1994,jacobs_bayesian_1997} and multigate MoE \cite{ma_modeling_2018,liang_m3vit_2022}, which have remained elusive in the literature.

\section*{Acknowledgements}
NH acknowledges support from the NSF IFML 2019844 and the NSF AI Institute for Foundations of Machine Learning. 

\section*{Impact Statement}

This paper presents work whose goal is to advance the field of Machine Learning. 
Given the theoretical nature of the paper, we believe that there are no potential societal consequences of our work.

\bibliography{icml_references}

\begin{thebibliography}{39}
\providecommand{\natexlab}[1]{#1}
\providecommand{\url}[1]{\texttt{#1}}
\expandafter\ifx\csname urlstyle\endcsname\relax
  \providecommand{\doi}[1]{doi: #1}\else
  \providecommand{\doi}{doi: \begingroup \urlstyle{rm}\Url}\fi

\bibitem[Chen et~al.(2022)Chen, Deng, Wu, Gu, and Li]{chen2022theory}
Chen, Z., Deng, Y., Wu, Y., Gu, Q., and Li, Y.
\newblock Towards understanding the mixture-of-experts layer in deep learning.
\newblock In Koyejo, S., Mohamed, S., Agarwal, A., Belgrave, D., Cho, K., and
  Oh, A. (eds.), \emph{Advances in Neural Information Processing Systems},
  volume~35, pp.\  23049--23062. Curran Associates, Inc., 2022.

\bibitem[Chow et~al.(2023)Chow, Tulepbergenov, Nachum, Gupta, Ryu, Ghavamzadeh,
  and Boutilier]{chow_mixture_expert_2023}
Chow, Y., Tulepbergenov, A., Nachum, O., Gupta, D., Ryu, M., Ghavamzadeh, M.,
  and Boutilier, C.
\newblock A {Mixture}-of-{Expert} {Approach} to {RL}-based {Dialogue}
  {Management}.
\newblock In \emph{The {Eleventh} {International} {Conference} on {Learning}
  {Representations}}, 2023.
\newblock URL \url{https://openreview.net/forum?id=4FBUihxz5nm}.

\bibitem[Csord{\'a}s et~al.(2023)Csord{\'a}s, Irie, and
  Schmidhuber]{csordas2023approximating}
Csord{\'a}s, R., Irie, K., and Schmidhuber, J.
\newblock Approximating two-layer feedforward networks for efficient
  transformers.
\newblock In Bouamor, H., Pino, J., and Bali, K. (eds.), \emph{Findings of the
  Association for Computational Linguistics: EMNLP 2023}, pp.\  674--692,
  Singapore, December 2023. Association for Computational Linguistics.

\bibitem[{De Veaux}(1989)]{deveaux1989linear}
{De Veaux}, R.~D.
\newblock Mixtures of linear regressions.
\newblock \emph{Computational Statistics and Data Analysis}, 8\penalty0
  (3):\penalty0 227--245, 1989.
\newblock ISSN 0167-9473.

\bibitem[Du et~al.(2022)Du, Huang, Dai, Tong, Lepikhin, Xu, Krikun, Zhou, Yu,
  Firat, Zoph, Fedus, Bosma, Zhou, Wang, Wang, Webster, Pellat, Robinson,
  Meier-Hellstern, Duke, Dixon, Zhang, Le, Wu, Chen, and Cui]{Du_Glam_MoE}
Du, N., Huang, Y., Dai, A.~M., Tong, S., Lepikhin, D., Xu, Y., Krikun, M.,
  Zhou, Y., Yu, A., Firat, O., Zoph, B., Fedus, L., Bosma, M., Zhou, Z., Wang,
  T., Wang, E., Webster, K., Pellat, M., Robinson, K., Meier-Hellstern, K.,
  Duke, T., Dixon, L., Zhang, K., Le, Q., Wu, Y., Chen, Z., and Cui, C.
\newblock Glam: Efficient scaling of language models with mixture-of-experts.
\newblock In \emph{ICML}, 2022.

\bibitem[Faria \& Soromenho(2010)Faria and Soromenho]{faria2010regression}
Faria, S. and Soromenho, G.
\newblock Fitting mixtures of linear regressions.
\newblock \emph{Journal of Statistical Computation and Simulation}, 80\penalty0
  (2):\penalty0 201--225, 2010.

\bibitem[Fedus et~al.(2022{\natexlab{a}})Fedus, Dean, and
  Zoph]{fedus2022review}
Fedus, W., Dean, J., and Zoph, B.
\newblock A review of sparse expert models in deep learning.
\newblock \emph{arXiv preprint arxiv 2209.01667}, 2022{\natexlab{a}}.

\bibitem[Fedus et~al.(2022{\natexlab{b}})Fedus, Zoph, and
  Shazeer]{fedus2021switch}
Fedus, W., Zoph, B., and Shazeer, N.
\newblock Switch transformers: Scaling to trillion parameter models with simple
  and efficient sparsity.
\newblock \emph{Journal of Machine Learning Research}, 23:\penalty0 1--39,
  2022{\natexlab{b}}.

\bibitem[Gupta et~al.(2022)Gupta, Mukherjee, Subudhi, Gonzalez, Jose,
  Awadallah, and Gao]{gupta2022sparsely}
Gupta, S., Mukherjee, S., Subudhi, K., Gonzalez, E., Jose, D., Awadallah, A.,
  and Gao, J.
\newblock Sparsely activated mixture-of-experts are robust multi-task learners.
\newblock \emph{arXiv preprint arxiv 2204.0768}, 2022.

\bibitem[Han et~al.(2024)Han, Nguyen, Harris, Ho, and Saria]{han2024fusemoe}
Han, X., Nguyen, H., Harris, C., Ho, N., and Saria, S.
\newblock Fusemoe: Mixture-of-experts transformers for fleximodal fusion.
\newblock \emph{arXiv preprint arXiv:2402.03226}, 2024.

\bibitem[Hazimeh et~al.(2021)Hazimeh, Zhao, Chowdhery, Sathiamoorthy, Chen,
  Mazumder, Hong, and Chi]{hazimeh_dselect_k_2021}
Hazimeh, H., Zhao, Z., Chowdhery, A., Sathiamoorthy, M., Chen, Y., Mazumder,
  R., Hong, L., and Chi, E.
\newblock {DSelect}-k: {Differentiable} {Selection} in the {Mixture} of
  {Experts} with {Applications} to {Multi}-{Task} {Learning}.
\newblock In Ranzato, M., Beygelzimer, A., Dauphin, Y., Liang, P.~S., and
  Vaughan, J.~W. (eds.), \emph{Advances in {Neural} {Information} {Processing}
  {Systems}}, volume~34, pp.\  29335--29347. Curran Associates, Inc., 2021.

\bibitem[Hendrycks \& Gimpel(2023)Hendrycks and Gimpel]{hendrycks2023gaussian}
Hendrycks, D. and Gimpel, K.
\newblock Gaussian error linear units (gelus).
\newblock \emph{arXiv preprint arxiv 1606.08415}, 2023.

\bibitem[Ho et~al.(2022)Ho, Yang, and Jordan]{ho2022gaussian}
Ho, N., Yang, C.-Y., and Jordan, M.~I.
\newblock Convergence rates for {G}aussian mixtures of experts.
\newblock \emph{Journal of Machine Learning Research}, 23\penalty0
  (323):\penalty0 1--81, 2022.

\bibitem[Jacobs et~al.(1991)Jacobs, Jordan, Nowlan, and
  Hinton]{Jacob_Jordan-1991}
Jacobs, R.~A., Jordan, M.~I., Nowlan, S.~J., and Hinton, G.~E.
\newblock Adaptive mixtures of local experts.
\newblock \emph{Neural Computation}, 3, 1991.

\bibitem[Jacobs et~al.(1997)Jacobs, Peng, and Tanner]{jacobs_bayesian_1997}
Jacobs, R.~A., Peng, F., and Tanner, M.~A.
\newblock A {Bayesian} {Approach} to {Model} {Selection} in {Hierarchical}
  {Mixtures}-of-{Experts} {Architectures}.
\newblock \emph{Neural Networks}, 10\penalty0 (2):\penalty0 231--241, 1997.
\newblock ISSN 0893-6080.

\bibitem[Jiang et~al.(2024)Jiang, Sablayrolles, Roux, Mensch, Savary, Bamford,
  Chaplot, de~las Casas, Hanna, Bressand, Lengyel, Bour, Lample, Lavaud,
  Saulnier, Lachaux, Stock, Subramanian, Yang, Antoniak, Scao, Gervet, Lavril,
  Wang, Lacroix, and Sayed]{jiang2024mixtral}
Jiang, A.~Q., Sablayrolles, A., Roux, A., Mensch, A., Savary, B., Bamford, C.,
  Chaplot, D.~S., de~las Casas, D., Hanna, E.~B., Bressand, F., Lengyel, G.,
  Bour, G., Lample, G., Lavaud, L.~R., Saulnier, L., Lachaux, M.-A., Stock, P.,
  Subramanian, S., Yang, S., Antoniak, S., Scao, T.~L., Gervet, T., Lavril, T.,
  Wang, T., Lacroix, T., and Sayed, W.~E.
\newblock Mixtral of experts.
\newblock \emph{arxiv preprint arxiv 2401.04088}, 2024.

\bibitem[Jordan \& Jacobs(1994)Jordan and Jacobs]{Jordan-1994}
Jordan, M.~I. and Jacobs, R.~A.
\newblock Hierarchical mixtures of experts and the {EM} algorithm.
\newblock \emph{Neural Computation}, 6:\penalty0 181--214, 1994.

\bibitem[Li et~al.(2023)Li, Shen, Yang, Wang, Ren, Che, Zhang, and
  Liu]{li2023sparse}
Li, B., Shen, Y., Yang, J., Wang, Y., Ren, J., Che, T., Zhang, J., and Liu, Z.
\newblock Sparse mixture-of-experts are domain generalizable learners.
\newblock In \emph{The Eleventh International Conference on Learning
  Representations}, 2023.

\bibitem[Liang et~al.(2022)Liang, Fan, Sarkar, Jiang, Chen, Zou, Cheng, Hao,
  and Wang]{liang_m3vit_2022}
Liang, H., Fan, Z., Sarkar, R., Jiang, Z., Chen, T., Zou, K., Cheng, Y., Hao,
  C., and Wang, Z.
\newblock M$^3${ViT}: {Mixture}-of-{Experts} {Vision} {Transformer} for
  {Efficient} {Multi}-task {Learning} with {Model}-{Accelerator} {Co}-design.
\newblock In \emph{{NeurIPS}}, 2022.

\bibitem[Lindsay(1995)]{Lindsay-1995}
Lindsay, B.
\newblock \emph{Mixture models: Theory, geometry and applications}.
\newblock In NSF-CBMS Regional Conference Series in Probability and Statistics.
  IMS, Hayward, CA., 1995.

\bibitem[Ma et~al.(2018)Ma, Zhao, Yi, Chen, Hong, and Chi]{ma_modeling_2018}
Ma, J., Zhao, Z., Yi, X., Chen, J., Hong, L., and Chi, E.~H.
\newblock Modeling {Task} {Relationships} in {Multi}-{Task} {Learning} with
  {Multi}-{Gate} {Mixture}-of-{Experts}.
\newblock In \emph{Proceedings of the 24th {ACM} {SIGKDD} {International}
  {Conference} on {Knowledge} {Discovery} \& {Data} {Mining}}, {KDD} '18, pp.\
  1930--1939, New York, NY, USA, 2018. Association for Computing Machinery.
\newblock ISBN 978-1-4503-5552-0.
\newblock \doi{10.1145/3219819.3220007}.
\newblock URL \url{https://doi.org/10.1145/3219819.3220007}.
\newblock event-place: London, United Kingdom.

\bibitem[Makkuva et~al.(2019)Makkuva, Viswanath, Kannan, and
  Oh]{makkuva19gridlock}
Makkuva, A., Viswanath, P., Kannan, S., and Oh, S.
\newblock Breaking the gridlock in mixture-of-experts: Consistent and efficient
  algorithms.
\newblock In \emph{Proceedings of the 36th International Conference on Machine
  Learning}, volume~97 of \emph{Proceedings of Machine Learning Research}, pp.\
   4304--4313. PMLR, 09--15 Jun 2019.

\bibitem[Manole \& Ho(2022)Manole and Ho]{manole22refined}
Manole, T. and Ho, N.
\newblock Refined convergence rates for maximum likelihood estimation under
  finite mixture models.
\newblock In \emph{Proceedings of the 39th International Conference on Machine
  Learning}, volume 162 of \emph{Proceedings of Machine Learning Research},
  pp.\  14979--15006. PMLR, 17--23 Jul 2022.

\bibitem[McLachlan \& Basford(1988)McLachlan and Basford]{Mclachlan-1988}
McLachlan, G.~J. and Basford, K.~E.
\newblock \emph{Mixture models: Inference and Applications to Clustering.
  Statistics: Textbooks and Monographs.}
\newblock New York, 1988.

\bibitem[Mendes \& Jiang(2011)Mendes and Jiang]{mendes2011convergence}
Mendes, E.~F. and Jiang, W.
\newblock Convergence rates for mixture-of-experts.
\newblock \emph{arXiv preprint arxiv 1110.2058}, 2011.

\bibitem[Nguyen et~al.(2023{\natexlab{a}})Nguyen, Akbarian, Nguyen, and
  Ho]{nguyen_general_2023}
Nguyen, H., Akbarian, P., Nguyen, T., and Ho, N.
\newblock A general theory for softmax gating multinomial logistic mixture of
  experts.
\newblock \emph{arXiv preprint arXiv:2310.14188}, 2023{\natexlab{a}}.

\bibitem[Nguyen et~al.(2023{\natexlab{b}})Nguyen, Nguyen, and
  Ho]{nguyen2023demystifying}
Nguyen, H., Nguyen, T., and Ho, N.
\newblock Demystifying softmax gating function in {G}aussian mixture of
  experts.
\newblock In \emph{Advances in Neural Information Processing Systems},
  2023{\natexlab{b}}.

\bibitem[Nguyen et~al.(2024{\natexlab{a}})Nguyen, Akbarian, Yan, and
  Ho]{nguyen2024statistical}
Nguyen, H., Akbarian, P., Yan, F., and Ho, N.
\newblock Statistical perspective of top-k sparse softmax gating mixture of
  experts.
\newblock In \emph{International Conference on Learning Representations},
  2024{\natexlab{a}}.

\bibitem[Nguyen et~al.(2024{\natexlab{b}})Nguyen, Nguyen, Nguyen, and
  Ho]{nguyen2024gaussian}
Nguyen, H., Nguyen, T., Nguyen, K., and Ho, N.
\newblock Towards convergence rates for parameter estimation in
  {Gaussian}-gated mixture of experts.
\newblock In \emph{Proceedings of The 27th International Conference on
  Artificial Intelligence and Statistics}, 2024{\natexlab{b}}.

\bibitem[Nie et~al.(2022)Nie, Miao, Cao, Ma, Liu, Xue, Miao, Liu, Yang, and
  Cui]{nie2022evomoe}
Nie, X., Miao, X., Cao, S., Ma, L., Liu, Q., Xue, J., Miao, Y., Liu, Y., Yang,
  Z., and Cui, B.
\newblock Evomoe: An evolutional mixture-of-experts training framework via
  dense-to-sparse gate.
\newblock \emph{arXiv preprint arXiv:2112.14397}, 2022.

\bibitem[Pham et~al.(2024)Pham, Do, Nguyen, Nguyen, Liu, Sartipi, Nguyen,
  Ramasamy, Li, Hoi, and Ho]{pham2024competesmoe}
Pham, Q., Do, G., Nguyen, H., Nguyen, T., Liu, C., Sartipi, M., Nguyen, B.~T.,
  Ramasamy, S., Li, X., Hoi, S., and Ho, N.
\newblock Competesmoe -- effective training of sparse mixture of experts via
  competition.
\newblock \emph{arXiv preprint arXiv:2402.02526}, 2024.

\bibitem[Puigcerver et~al.(2024)Puigcerver, Riquelme, Mustafa, and
  Houlsby]{puigcerver2024sparse}
Puigcerver, J., Riquelme, C., Mustafa, B., and Houlsby, N.
\newblock From sparse to soft mixtures of experts.
\newblock In \emph{The Twelfth International Conference on Learning
  Representations}, 2024.

\bibitem[Riquelme et~al.(2021)Riquelme, Puigcerver, Mustafa, Neumann, Jenatton,
  Pint, Keysers, and Houlsby]{Riquelme2021scalingvision}
Riquelme, C., Puigcerver, J., Mustafa, B., Neumann, M., Jenatton, R., Pint,
  A.~S., Keysers, D., and Houlsby, N.
\newblock Scaling vision with sparse mixture of experts.
\newblock In \emph{Advances in Neural Information Processing Systems},
  volume~34, pp.\  8583--8595. Curran Associates, Inc., 2021.

\bibitem[Ruiz et~al.(2021)Ruiz, Puigcerver, Mustafa, Neumann, Jenatton, Pinto,
  Keysers, and Houlsby]{Ruiz_Vision_MoE}
Ruiz, C., Puigcerver, J., Mustafa, B., Neumann, M., Jenatton, R., Pinto, A.,
  Keysers, D., and Houlsby, N.
\newblock Scaling vision with sparse mixture of experts.
\newblock In \emph{NeurIPS}, 2021.

\bibitem[Shazeer et~al.(2017)Shazeer, Mirhoseini, Maziarz, Davis, Le, Hinton,
  and Dean]{shazeer2017topk}
Shazeer, N., Mirhoseini, A., Maziarz, K., Davis, A., Le, Q., Hinton, G., and
  Dean, J.
\newblock Outrageously large neural networks: The sparsely-gated
  mixture-of-experts layer.
\newblock In \emph{In International Conference on Learning Representations},
  2017.

\bibitem[van~de Geer(2000)]{vandeGeer-00}
van~de Geer, S.
\newblock \emph{Empirical processes in M-estimation}.
\newblock Cambridge University Press, 2000.

\bibitem[Yu(1997)]{yu97lecam}
Yu, B.
\newblock Assouad, {F}ano, and {L}e {C}am.
\newblock \emph{Festschrift for Lucien Le Cam}, pp.\  423--435, 1997.

\bibitem[Zhao et~al.(1994)Zhao, Schwartz, Sroka, and
  Makhoul]{zhao_hierarchical_1994}
Zhao, Y., Schwartz, R., Sroka, J., and Makhoul, J.
\newblock Hierarchical mixtures of experts methodology applied to continuous
  speech recognition.
\newblock In \emph{Advances in {Neural} {Information} {Processing} {Systems}},
  volume~7, 1994.

\bibitem[Zhou et~al.(2023)Zhou, Du, Huang, Peng, Lan, Huang, Shakeri, So, Dai,
  Lu, Chen, Le, Cui, Laudon, and Dean]{zhou2023brainformers}
Zhou, Y., Du, N., Huang, Y., Peng, D., Lan, C., Huang, D., Shakeri, S., So, D.,
  Dai, A., Lu, Y., Chen, Z., Le, Q., Cui, C., Laudon, J., and Dean, J.
\newblock Brainformers: Trading simplicity for efficiency.
\newblock In \emph{International Conference on Machine Learning}, pp.\
  42531--42542. PMLR, 2023.

\end{thebibliography}
\bibliographystyle{icml2024}

\newpage
\appendix
\onecolumn
\centering
\textbf{\Large{Supplementary Material for \\ \vspace{.2em} 
``On Least Square Estimation in Softmax Gating Mixture of Experts''}}

\justifying
\setlength{\parindent}{0pt}
In this supplementary material, we provide proofs for the main results in the paper in Appendix~\ref{sec:proofs}, while we leave proofs for the identifiability of the softmax gating mixture of experts in Appendix~\ref{appendix:identifiability}. Finally, we run several numerical experiments in Appendix~\ref{sec:experiments} to empirically justify our theoretical results.
\section{Proofs of Main Results}
\label{sec:proofs}
In this appendix, we provide proofs for main results in the paper. 
\subsection{Proof of Theorem~\ref{theorem:l2_bound}}
\label{appendix:l2_bound}
For the proof of the theorem, we first introduce some notation. Firstly, we denote by $\mathcal{F}_k(\Theta)$ the set of regression functions w.r.t all mixing measures in $\mathcal{G}_k(\Theta)$, that is, $\mathcal{F}_k(\Theta):=\{f_{G}(x):G\in\mathcal{G}_{k}(\Theta)\}$.
Additionally, for each $\delta>0$, the $L^{2}$ ball centered around the regression function $f_{G_*}(x)$ and intersected with the set $ \mathcal{F}_k(\Theta)$ is defined as
\begin{align*}   \mathcal{F}_k(\Theta,\delta):=\left\{f \in \mathcal{F}_k(\Theta): \|f -f_{G_*}\|_{L^2(\mu)} \leq\delta\right\}.
\end{align*}
In order to measure the size of the above set, Geer et. al. \cite{vandeGeer-00} suggest using the following quantity:
\begin{align}
    \label{eq:bracket_size}
    \mathcal{J}_B(\delta, \mathcal{F}_k(\Theta,\delta)):=\int_{\delta^2/2^{13}}^{\delta}H_B^{1/2}(t, \mathcal{F}_k(\Theta,t),\|\cdot\|_{L^2(\mu)})~\dint t\vee \delta,
\end{align}
where $H_B(t, \mathcal{F}_k(\Theta,t),\|\cdot\|_{L^2(\mu)})$ stands for the bracketing entropy \cite{vandeGeer-00} of $ \mathcal{F}_k(\Theta,u)$ under the $L^{2}$-norm, and $t\vee\delta:=\max\{t,\delta\}$. By using the similar proof argument of Theorem 7.4 and Theorem 9.2 in \cite{vandeGeer-00} with notations being adapted to this work, we obtain the following lemma:
\begin{lemma}
    \label{lemma:density_rate}
    Take $\Psi(\delta)\geq \mathcal{J}_B(\delta, \mathcal{F}_k(\Theta,\delta))$ that satisfies $\Psi(\delta)/\delta^2$ is a non-increasing function of $\delta$. Then, for some universal constant $c$ and for some sequence $(\delta_n)$ such that $\sqrt{n}\delta^2_n\geq c\Psi(\delta_n)$, we achieve that
    \begin{align*}
        \mathbb{P}\Big(\|f_{\widehat{G}_n} - f_{G_*}\|_{L^2(\mu)} > \delta\Big)\leq c \exp\left(-\frac{n\delta^2}{c^2}\right),
    \end{align*}
    for all $\delta\geq \delta_n$.
\end{lemma}
We now demonstrate that when the expert functions are Lipschitz continuous, the following bound holds:
\begin{align}    
H_B(\varepsilon,\mathcal{F}_k(\Theta),\|.\|_{L^{2}(\mu)}) \lesssim \log(1/\varepsilon), \label{eq:bracket_entropy_bound}
\end{align}
for any $0 < \varepsilon \leq 1/2$. Indeed, for any function $f_{G} \in \mathcal{F}_k(\Theta)$, since the expert functions are bounded, we obtain that $f_{G}(x) \leq M$ for all $x$ where $M$ is bounded constant of the expert functions. Let $\tau\leq\varepsilon$ and $\{\pi_1,\ldots,\pi_N\}$ be the $\tau$-cover under the $L^{2}$ norm of the set $\mathcal{F}_k(\Theta)$ where $N:={N}(\tau,\mathcal{F}_k(\Theta),\|\cdot\|_{L^{2}(\mu)})$ is the $\eta$-covering number of the metric space $(\mathcal{F}_k(\Theta),\|\cdot\|_{L^{2}(\mu)})$. Then, we construct the brackets of the form $[L_i(x),U_i(x)]$ for all $i\in[N]$ as follows:
    \begin{align*}
        L_i(x)&:=\max\{\pi_i(x)-\tau,0\},\\
        U_i(x)&:=\max\{\pi_i(x)+\tau, M \}.
    \end{align*}
From the above construction, we can validate that $\mathcal{F}_{k}(\Theta)\subset\cup_{i=1}^{N}[L_i(x),U_i(x)]$ and $U_i(x)-L_i(x)\leq 2\min\{2\tau,M\}$. Therefore, it follows that 
\begin{align*}
    \normf{U_i-L_i}^2=\int(U_i-L_i)^2\dint\mu(x)\leq\int 16\tau^2\dint\mu(x)=16\tau^2,
\end{align*}
which implies that $\normf{U_i-L_i}\leq 4\tau$. By definition of the bracketing entropy, we deduce that
\begin{align}
    \label{eq:bracketing_covering}
    H_B(4\tau,\mathcal{F}_{k}(\Theta),\normf{\cdot})\leq\log N=\log {N}(\tau,\mathcal{F}_k(\Theta),\|\cdot\|_{L^{2}(\mu)}).
\end{align}
Therefore, we need to provide an upper bound for the covering number $N$. In particular, we denote $\Delta:=\{(\beta_0,\beta_1)\in\mathbb{R}\times\mathbb{R}^d:(\beta_0,\beta_1,\eta)\in\Theta\}$ and $\Omega:=\{\eta\in\mathbb{R}^q:(\beta_{0},\beta_{1},\eta)\in\Theta\}$. Since $\Theta$ is a compact set, $\Delta$ and $\Omega$ are also compact. Therefore, we can find $\tau$-covers $\Delta_{\tau}$ and ${\Omega}_{\tau}$ for $\Delta$ and $\Omega$, respectively. We can check that 
\begin{align*}
    |\Delta_{\tau}|\leq \mathcal{O}_{P}(\tau^{-(d+1)k}), \quad |\Omega_{\tau}|\lesssim \mathcal{O}_{P}(\tau^{-qk}).
\end{align*}
For each mixing measure $G=\sum_{i=1}^{k}\exp(\beta_{0i})\delta_{(\beta_{1i},\eta_i)}\in\mathcal{G}_k(\Theta)$, we consider other two mixing measures:
\begin{align*}
    \widetilde{G}:=\sum_{i=1}^k\exp(\beta_{0i})\delta_{({\beta}_{1i},\overline{\eta}_i)}, \qquad \overline{G}:=\sum_{i=1}^k\exp(\overline{\beta}_{0i})\delta_{({\overline{\beta}}_{1i},\overline{\eta}_i)}.
\end{align*}
Here, $\overline{\eta}_i\in{\Omega}_{\tau}$ such that $\overline{\eta}_i$ is the closest to $\eta_i$ in that set, while $(\overline{\beta}_{0i},\overline{\beta}_{1i})\in\Delta_{\tau}$ is the closest to $(\beta_{0i},\beta_{1i})$ in that set. From the above formulations, we get that
\begin{align*}
    \normf{f_{G}-f_{\widetilde{G}}}^2&=\int\Bigg[\sum_{i=1}^{k}\frac{\exp((\beta_{1i})^{\top}x+\beta_{0i})}{\sum_{j=1}^{k}\exp((\beta_{1j})^{\top}x+\beta_{0j})}\cdot[h(x,\eta_i)-h(x,\overline{\eta}_i)\Bigg]^2~\dint\mu(x)\\
    &\leq k\int\sum_{i=1}^{k}\Bigg[\frac{\exp((\beta_{1i})^{\top}x+\beta_{0i})}{\sum_{j=1}^{k}\exp((\beta_{1j})^{\top}x+\beta_{0j})}\cdot[h(x,\eta_i)-h(x,\overline{\eta}_i)\Bigg]^2~\dint\mu(x)\\
    &\leq k\int\sum_{i=1}^{k}~[h(x,\eta_i)-h(x,\overline{\eta}_i)]^2~\dint\mu(x)\\
    &\leq k\int\sum_{i=1}^{k}~[L_1\cdot\|\eta_i-\overline{\eta}_i\|]^2~\dint\mu(x)\\
    &\leq k^2(L_1\tau)^2,
\end{align*}
which indicates that $\normf{f_{G}-f_{\widetilde{G}}}\leq L_1k\tau$. Here, the second inequality is according to the Cauchy-Schwarz inequality, the third inequality occurs as the softmax weight is bounded by 1, and the fourth inequality follows from the fact that the expert $h(x,\cdot)$ is a Lipschitz function with Lipschitz constant $L_1$. Next, we have
\begin{align*}
    \normf{f_{\widetilde{G}}-f_{\overline{G}}}^2&\leq k\int\sum_{i=1}^{k}\Bigg[\Bigg(\frac{\exp((\beta_{1i})^{\top}x+\beta_{0i})}{\sum_{j=1}^{k}\exp((\beta_{1j})^{\top}x+\beta_{0j})}-\frac{\exp((\overline{\beta}_{1i})^{\top}x+\overline{\beta}_{0i})}{\sum_{j=1}^{k}\exp((\overline{\beta}_{1j})^{\top}x+\overline{\beta}_{0j})}\Bigg)\cdot h(x,\overline{\eta}_i)\Bigg]^2~\dint\mu(x)\\
    &\leq kM^2L^2\int\sum_{i=1}^{k}\Big[\|\beta_{1i}-\overline{\beta}_{1i}\|\cdot\|x\|+|\beta_{0i}-\overline{\beta}_{0i}|\Big]^2\dint\mu(x)\\
    &\leq kM^2L^2\int\sum_{i=1}^{k}(\tau\cdot B+\tau)^2\dint\mu(x)\\
    &\leq [kML\tau(B+1)]^2,
\end{align*}
where $L\geq 0$ is a Lipschitz constant of the softmax weight. This result implies that $ \normf{f_{\widetilde{G}}-f_{\overline{G}}}\leq kML(B+1)\tau$. According to the triangle inequality, we have
\begin{align*}
    \normf{f_{G}-f_{\overline{G}}}\leq \normf{f_{G}-f_{\widetilde{G}}}+\normf{f_{\widetilde{G}}-f_{\overline{G}}}\leq [L_1k+kML(B+1)]\cdot\tau.
\end{align*}
By definition of the covering number, we deduce that
\begin{align}
    \label{eq:covering_bound}
    {N}(\tau,\mathcal{F}_k(\Theta),\normf{\cdot})\leq |\Delta_{\tau}|\times|\Omega_{\tau}|\leq \mathcal{O}_{P}(n^{-(d+1)k})\times\mathcal{O}(n^{-qk})\leq\mathcal{O}(n^{-(d+1+q)k}).
\end{align}
Combine equations~\eqref{eq:bracketing_covering} and \eqref{eq:covering_bound}, we achieve that
\begin{align*}
    H_B(4\tau,\mathcal{F}_{k}(\Theta),\normf{\cdot})\lesssim \log(1/\tau).
\end{align*}
Let $\tau=\varepsilon/4$, then we obtain that 
\begin{align*}
    H_B(\varepsilon,\mathcal{F}_k(\Theta),\|.\|_{L^{2}(\mu)}) \lesssim \log(1/\varepsilon).
\end{align*}
As a result, it follows that 
\begin{align}
    \label{eq:bracketing_integral}
    \mathcal{J}_B(\delta, \mathcal{F}_k(\Theta,\delta))= \int_{\delta^2/2^{13}}^{\delta}H_B^{1/2}(t, \mathcal{F}_k(\Theta,t),\normf{\cdot})~\dint t\vee \delta\lesssim \int_{\delta^2/2^{13}}^{\delta}\log(1/t)dt\vee\delta.
\end{align}
Let $\Psi(\delta)=\delta\cdot[\log(1/\delta)]^{1/2}$, then $\Psi(\delta)/\delta^2$ is a non-increasing function of $\delta$. Furthermore, equation~\eqref{eq:bracketing_integral} indicates that $\Psi(\delta)\geq \mathcal{J}_B(\delta,\mathcal{F}_k(\Theta,\delta))$. In addition, let $\delta_n=\sqrt{\log(n)/n}$, then we get that $\sqrt{n}\delta^2_n\geq c\Psi(\delta_n)$ for some universal constant $c$. Finally, by applying Lemma~\ref{lemma:density_rate}, we achieve the desired conclusion of the theorem.

\subsection{Proof of Theorem~\ref{theorem:general_experts}}
\label{appendix:general_experts}
In this proof, we aim to establish the following inequality:
\begin{align}
    \label{eq:general_universal_inequality}
    \inf_{G\in\mathcal{G}_{k}(\Theta)}\normf{f_{G}-f_{G_*}}/\mathcal{D}_1(G,G_*)>0.
\end{align}
For that purpose, we divide the proof of the above inequality into local and global parts in the sequel.

\textbf{Local part:} In this part, we demonstrate that
\begin{align}
    \label{eq:general_local_inequality}
    \lim_{\varepsilon\to0}\inf_{G\in\mathcal{G}_{k}(\Theta):\mathcal{D}_1(G,G_*)\leq\varepsilon}\normf{f_{G}-f_{G_*}}/\mathcal{D}_1(G,G_*)>0.
\end{align}
Assume by contrary that the above inequality does not hold true, then there exists a sequence of mixing measures $G_n=\sum_{i=1}^{k_*}\exp(\beta^n_{0i})\delta_{(\beta^n_{1i},\eta^n_i)}$ in $\mathcal{G}_{k}(\Theta)$ such that $\mathcal{D}_{1n}:=\mathcal{D}_1(G_n,G_*)\to0$ and
\begin{align}
    \label{eq:general_ratio_limit}
    \normf{f_{G_n}-f_{G_*}}/\mathcal{D}_{1n}\to0,
\end{align}
as $n\to\infty$. Let us denote by $\mathcal{A}^n_j:=\mathcal{A}_j(G_n)$ a Voronoi cell of $G_n$ generated by the $j$-th components of $G_*$. Since our arguments are asymptotic, we may assume that those Voronoi cells do not depend on the sample size, i.e. $\mathcal{A}_j=\mathcal{A}^n_j$. Thus, the Voronoi loss $\mathcal{D}_{1n}$ can be represented as
\begin{align}
    \label{eq:general_loss_proof}
   \mathcal{D}_{1n}&:=\sum_{j=1}^{k_*}\Big|\sum_{i\in\mathcal{A}_j}\exp(\bzin)-\exp(\beta^*_{0j})\Big|+\sum_{j:|\mathcal{A}_j|>1}\sum_{i\in\mathcal{A}_j}\exp(\bzin)\Big[\|\dboijn\|^2+\|\deijn\|^2\Big]\nonumber\\
    &\hspace{6cm}+\sum_{j:|\mathcal{A}_j|=1}\sum_{i\in\mathcal{A}_j}\exp(\beta^n_{0i})\Big[\|\dboijn\|+\|\deijn\|\Big],
\end{align}
where we denote $\dboijn:=\beta^n_{1i}-\beta^*_{1j}$ and $\deijn:=\eta^n_i-\eta^*_j$.

Since $\mathcal{D}_{1n}\to0$, we get that $(\boin,\ein)\to(\boj,\ej)$ and $\sum_{i\in\mathcal{A}_j}\exp(\bzin)\to\exp(\bzj)$ as $n\to\infty$ for any $i\in\mathcal{A}_j$ and $j\in[k_*]$. Now, we divide the proof of local part into three steps as follows:

\textbf{Step 1.} In this step, we decompose the term $Q_n(x):=[\sum_{j=1}^{k_*}\exp((\beta^*_{1j})^{\top}x+\beta^*_{0j})]\cdot[f_{G_n}(x)-f_{G_*}(x)]$ into a combination of linearly independent elements using Taylor expansion. In particular, let us denote $F(x;\beta_1,\eta):=\exp(\beta_1^{\top}x)h(x,\eta)$ and $H(x;\beta_1)=\exp(\beta_1^{\top}x)f_{G_n}(x)$, then we have
\begin{align}
    \label{eq:general_Q_n}
    Q_n(x)&=\sum_{j=1}^{k_*}\sum_{i\in\mathcal{A}_j}\exp(\beta^n_{0i})\Big[F(x;\boin,\ein)-F(x;\boj,\ej)\Big]\nonumber\\
    &-\sum_{j=1}^{k_*}\sum_{i\in\mathcal{A}_j}\exp(\beta^n_{0i})\Big[H(x;\boin)-H(x;\boj)\Big]\nonumber\\
    &+\sum_{j=1}^{k_*}\Big(\sum_{i\in\mathcal{A}_j}\exp(\bzin)-\exp(\bzj)\Big)\Big[F(x;\boj,\ej)-H(x;\boj)\Big]\nonumber\\
    &:=A_n(x)-B_n(x)+E_{n}(x).
\end{align}
\textbf{Decomposition of $A_n$.} Next, we continue to separate the term $A_n$ into two parts as follows:
\begin{align*}
    A_n(x)&:=\sum_{j:|\mathcal{A}_j|=1}\sum_{i\in\mathcal{A}_j}\exp(\beta^n_{0i})\Big[F(x;\boin,\ein)-F(x;\boj,\ej)\Big]\\
    &+\sum_{j:|\mathcal{A}_j|>1}\sum_{i\in\mathcal{A}_j}\exp(\beta^n_{0i})\Big[F(x;\boin,\ein)-F(x;\boj,\ej)\Big]\\
    &:=A_{n,1}(x)+A_{n,2}(x).
\end{align*}
By means of the first-order Taylor expansion, we have
\begin{align*}
    A_{n,1}(x)=\sum_{j:|\mathcal{A}_j|=1}\sum_{i\in\mathcal{A}_j}\exp(\beta^n_{0i})\sum_{|\alpha|=1}(\dboijn)^{\alpha_1}(\deijn)^{\alpha_2}\cdot\frac{\partial F}{\partial\beta_1^{\alpha_1}\partial \eta^{\alpha_2}}(x;\boj,\ej)+R_1(x),
\end{align*}
where $R_1(x)$ is a Taylor remainder such that $R_1(x)/\mathcal{D}_{1n}\to0$ as $n\to\infty$. By taking the first derivatives of $F$ w.r.t its parameters, we get
\begin{align*}
    \frac{\partial F}{\partial\beta_1}(x;\boj,\ej)&=x\exp((\boj)^{\top}x)h(x,\ej)=x\cdot F(x;\boj,\ej),\\
    \frac{\partial F}{\partial \eta}(x;\boj,\ej)&=\exp((\boj)^{\top}x)\cdot\frac{\partial h}{\partial\eta}(x,\ej):=F_1(x;\boj,\ej).
\end{align*}
Thus, we can rewrite $A_{n,1}(x)$ as
\begin{align}
    \label{eq:general_A_n_1}
    A_{n,1}(x)&=\sum_{j:|\mathcal{A}_j|=1}C_{n,1,j}(x)+R_1(x),
\end{align}
where
\begin{align*}
    C_{n,1,j}(x)=\sum_{i\in\mathcal{A}_j}\exp(\beta^n_{0i})\Big[(\dboijn)^{\top}x\cdot F(x;\boj,\ej)+(\deijn)^{\top} F_1(x;\boj,\ej)\Big].
\end{align*}
Next, by applying the second-order Taylor expansion, $A_{n,2}(x)$ can be represented as
\begin{align*}
    A_{n,2}(x)=\sum_{j:|\mathcal{A}_j|>1}\sum_{i\in\mathcal{A}_j}\exp(\beta^n_{0i})\sum_{|\alpha|=1}^{2}\frac{1}{\alpha!}(\dboijn)^{\alpha_1}(\deijn)^{\alpha_2}\cdot\frac{\partial^{|\alpha_1|+|\alpha_2|} F}{\partial\beta_1^{\alpha_1}\partial \eta^{\alpha_2}}(x;\boj,\ej)+R_2(x),
\end{align*}
where $R_2(x)$ is a Taylor remainder such that $R_2(x)/\mathcal{D}_{1n}\to0$ as $n\to\infty$. The second derivatives of $F$ w.r.t its parameters are given by
\begin{align*}
    \frac{\partial^2F}{\partial\beta\partial\beta^{\top}}(x;\boj,\ej)&=xx^{\top}\cdot F(x;\boj,\ej),\quad \frac{\partial^2F}{\partial\beta\partial\eta^{\top}}(x;\boj,\ej)=x\cdot [F_1(x;\boj,\ej)]^{\top},\\
    \frac{\partial^2F}{\partial\eta\partial\eta^{\top}}(x;\boj,\ej)&=\exp((\boj)^{\top}x)\cdot\frac{\partial^2 h}{\partial\eta\partial\eta^{\top}}:=F_2(x;\boj,\ej)
\end{align*}
Therefore, the term $A_{n,2}(x)$ becomes
\begin{align}
     \label{eq:general_A_n_2}
    A_{n,2}(x)=\sum_{j:|\mathcal{A}_j|>1}[C_{n,1,j}(x)+C_{n,2,j}(x)]+R_2(x),
\end{align}
where 
\begin{align*}
    &C_{n,2,j}(x):=\sum_{i\in\mathcal{A}_j}\exp(\bzin)\Bigg\{\Big[x^{\top}\Big(M_{d}\odot (\dboijn)(\dboijn)^{\top}\Big)x\Big]\cdot F(x;\boj,\aj,\bj)\\
    &+\Big[x^{\top}(\dboijn)(\deijn)^{\top}F_1(x;\boj,\ej)\Big]+\Big[(\deijn)^{\top}\Big(M_{d}\odot F_2(x;\boj,\ej)\Big)(\deijn)\Big]\Bigg\},
\end{align*}
with $M_d$ being an $d\times d$ matrix whose diagonal entries are $\frac{1}{2}$ while other entries are 1.

\textbf{Decomposition of $B_n$.} Subsequently, we also divide $B_n$ into two terms based on the Voronoi cells as

\begin{align*}
    B_n(x)&=\sum_{j:|\mathcal{A}_j|=1}\sum_{i\in\mathcal{A}_j}\exp(\beta^n_{0i})\Big[H(x;\boin)-H(x;\boj)\Big]\\
    &+\sum_{j:|\mathcal{A}_j|>1}\sum_{i\in\mathcal{A}_j}\exp(\beta^n_{0i})\Big[H(x;\boin)-H(x;\boj)\Big]\\
    &:=B_{n,1}(x)+B_{n,2}(x).
\end{align*}
By means of the first-order Taylor expansion, we have
\begin{align}
    \label{eq:general_B_n_1}
    B_{n,1}(x)=\sum_{j:|\mathcal{A}_j|=1}\sum_{i\in\mathcal{A}_j}\exp(\bzin)(\dboijn)^{\top}x\cdot H(x;\boj)+R_3(x),
\end{align}
where $R_3(x)$ is a Taylor remainder such that $R_3(x)/\mathcal{D}_{1n}\to0$ as $n\to\infty$. Meanwhile, by applying the second-order Taylor expansion, we get
\begin{align}
      \label{eq:general_B_n_2}
    B_{n,2}(x)=\sum_{j:|\mathcal{A}_j|>1}\sum_{i\in\mathcal{A}_j}\exp(\bzin)\Big[(\dboijn)^{\top}x+(\dboijn)^{\top}\Big(M_{d}\odot xx^{\top}\Big)(\dboijn)\Big]\cdot H(x;\boj) + R_4(x),
\end{align}
where $R_4(x)$ is a Taylor remainder such that $R_4(x)/\mathcal{D}_{1n}\to0$ as $n\to\infty$.

Putting the above results together, we see that $[A_n(x)-R_1(x)-R_2(x)]/\mathcal{D}_{1n}$, $[B_n(x)-R_3(x)-R_4(x)]/\mathcal{D}_{1n}$ and $E_n(x)/\mathcal{D}_{1n}$ can be written as a combination of elements from the following set
\begin{align*}
    &\Big\{F(x;\boj,\ej), \ x^{(u)}F(x;\boj,\ej), \ x^{(u)}x^{(v)}F(x;\boj,\ej):u,v\in[d], \ j\in[k_*]\Big\},\\
    \cup&~\Big\{[F_1(x;\boj,\ej)]^{(u)}, \ x^{(u)}[F_1(x;\boj,\ej)]^{(v)}:u,v\in[d], \ j\in[k_*]\Big\},\\
    \cup&~\Big\{[F_2(x;\boj,\ej)]^{(uv)}:u,v\in[d], \ j\in[k_*]\Big\},\\
    \cup&~\Big\{H(x;\boj), \ x^{(u)}H(x;\boj), \ x^{(u)}x^{(v)}H(x;\boj):u,v\in[d], \ j\in[k_*]\Big\}.
\end{align*}
\textbf{Step 2.} In this step, we prove by contradiction that at least one among coefficients in the representations of $[A_n-R_1(x)-R_2(x)]/\mathcal{D}_{2n}$, $[B_n-R_3(x)-R_4(x)]/\mathcal{D}_{2n}$ and $E_n(x)/\mathcal{D}_{2n}$ does not go to zero as $n$ tends to infinity. Indeed, assume that all of them converge to zero. Then, by considering the coefficients of 
\begin{itemize}
    \item $F(x;\boj,\ej)$ for $j\in[k_*]$, we get that $\frac{1}{\mathcal{D}_{1n}}\cdot\sum_{j=1}^{k_*}\Big|\sum_{i\in\mathcal{A}_j}\exp(\bzin)-\exp(\bzj)\Big|\to0$;
    \item $x^{(u)}F(x;\boj,\ej)$ for $u\in[d]$ and $j:|\mathcal{A}_j|=1$, we get that $\frac{1}{\mathcal{D}_{1n}}\cdot\sum_{j:|\mathcal{A}_j|=1}\sum_{i\in\mathcal{A}_j}\exp(\bzin)\|\dboijn\|_1\to0$;
    \item $[F_1(x;\boj,\ej)]^{(u)}$ for $u\in[d]$ and $j:|\mathcal{A}_j|=1$, we get that $\frac{1}{\mathcal{D}_{1n}}\cdot\sum_{j:|\mathcal{A}_j|=1}\sum_{i\in\mathcal{A}_j}\exp(\bzin)\|\deijn\|_1\to0$;
    \item $[x^{(u)}]^2F(x;\boj,\ej)$ for $u\in[d]$ and $j:|\mathcal{A}_j|>1$, we get that $\frac{1}{\mathcal{D}_{1n}}\cdot\sum_{j:|\mathcal{A}_j|>1}\sum_{i\in\mathcal{A}_j}\exp(\bzin)\|\dboijn\|^2\to0$;
    \item $[F_2(x;\boj,\ej)]^{(uu)}$ for $u\in[d]$ and $j:|\mathcal{A}_j|>1$, we get that $\frac{1}{\mathcal{D}_{1n}}\cdot\sum_{j:|\mathcal{A}_j|>1}\sum_{i\in\mathcal{A}_j}\exp(\bzin)\|\deijn\|^2\to0$;
\end{itemize}
By taking the summation of the above limits, we obtain that $1=\mathcal{D}_{1n}/\mathcal{D}_{1n}\to0$ as $n\to\infty$, which is a contradiction. Therefore, not all the coefficients in the representations of $[A_n(x)-R_1(x)-R_2(x)]/\mathcal{D}_{1n}$, $[B_n(x)-R_3(x)-R_4(x)]/\mathcal{D}_{1n}$ and $E_n(x)/\mathcal{D}_{1n}$ go to zero.

\textbf{Step 3.} In this step, we point out a contradiction following from the result in Step 2. Let us denote by $m_n$ the maximum of the absolute values of the coefficients in the representations of $[A_n(x)-R_1(x)-R_2(x)]/\mathcal{D}_{1n}$, $[B_n(x)-R_3(x)-R_4(x)]/\mathcal{D}_{1n}$ and $E_n(x)/\mathcal{D}_{1n}$. Since at least one among those coefficients does not approach zero, we obtain that $1/m_n\not\to\infty$. 

Recall the hypothesis in equation~\eqref{eq:general_ratio_limit} that $\normf{f_{G_n}-f_{G_*}}/\mathcal{D}_{1n}\to0$ as $n\to\infty$, which indicates that $\|f_{G_n}-f_{G_*}\|_{L^1(\mu)}/\mathcal{D}_{1n}\to0$. By means of the Fatou's lemma, we have
\begin{align*}
    0=\lim_{n\to\infty}\frac{\|f_{G_n}-f_{G_*}\|_{L^1(\mu)}}{m_n\mathcal{D}_{1n}}\geq \int \liminf_{n\to\infty}\frac{|f_{G_n}(x)-f_{G_*}(x)|}{m_n\mathcal{D}_{1n}}\dint\mu(x)\geq 0.
\end{align*}
This result implies that $[f_{G_n}(x)-f_{G_*}(x)]/[m_n\mathcal{D}_{1n}]\to0$ for almost every $x$. Since the term $\sum_{j=1}^{k_*}\exp((\beta^*_{1j})^{\top}x+\beta^*_{0j})$ is bounded, we deduce that $Q_n(x)/[m_n\mathcal{D}_{1n}]\to0$, or equivalently, 
\begin{align}
    \label{eq:general_zero_limit}
    \frac{1}{m_n\mathcal{D}_{1n}}\cdot\Big[(A_{n,1}(x)-R_1(x)+A_{n,2}(x)-R_2(x))-(B_{n,1}(x)-R_3(x)+B_{n,2}(x)-R_4(x))+E_n(x)\Big]\to0.
\end{align}
Let us denote
\begin{align*}
    \frac{1}{m_n\mathcal{D}_{1n}}\cdot\sum_{i\in\mathcal{A}_j}\exp(\bzin)(\dboijn)\to\phi_{1,j}&,\quad  \frac{1}{m_n\mathcal{D}_{1n}}\cdot\sum_{i\in\mathcal{A}_j}\exp(\bzin)(\dboijn)(\dboijn)^{\top}\to\phi_{2,j},\\
    \frac{1}{m_n\mathcal{D}_{1n}}\cdot\sum_{i\in\mathcal{A}_j}\exp(\bzin)(\deijn)\to\varphi_{1,j}&,\quad \frac{1}{m_n\mathcal{D}_{1n}}\cdot\sum_{i\in\mathcal{A}_j}\exp(\bzin)(\deijn)(\deijn)^{\top}\to\varphi_{2,j},\\
    \frac{1}{m_n\mathcal{D}_{1n}}\cdot\sum_{i\in\mathcal{A}_j}\exp(\bzin)(\dboijn)(\deijn)^{\top}\to\zeta_{j}&, \quad \frac{1}{m_n\mathcal{D}_{1n}}\cdot\Big(\sum_{i\in\mathcal{A}_j}\exp(\bzin)-\exp(\bzj)\Big)\to\xi_j&.
\end{align*}
Here, at least one among $\phi^{(u)}_{1,j}$, $\phi^{(uu)}_{2,j}$, $\varphi^{(u)}_{1,j}$, $\varphi^{(uu)}_{2,j}$ and $\xi_j$, for $j\in[k_*]$, is different from zero, which results from Step 2. Additionally, let us denote $F_{\tau j}:=F_{\tau}(x;\boj,\ej)$ and $H_j=H(x;\boj)$ for short, then from the formulation of 
\begin{itemize}
    \item $A_{n,1}$ in equation~\eqref{eq:general_A_n_1}, we get 
    \begin{align}
        \label{eq:general_limit_1}
        \frac{A_{n,1}-R_1(x)}{m_n\mathcal{D}_{1n}}\to\sum_{j:|\mathcal{A}_j|=1}\Big[\phi^{\top}_{1,j}x\cdot F_j+\varphi_{1,j}^{\top}F_{1j}\Big].
    \end{align}
    \item $A_{n,2}$ in equation~\eqref{eq:general_A_n_2}, we get
    \begin{align}
    \label{eq:general_limit_2}
        &\frac{A_{n,2}-R_2(x)}{m_n\mathcal{D}_{2n}}\to\sum_{j:|\mathcal{A}_j|>1}\Bigg\{\Big[\phi^{\top}_{1,j}x+x^{\top}\Big(M_{d}\odot \phi_{2,j}\Big)x\Big]\cdot F_j+[\varphi_{1,j}^{\top}+x^{\top}\zeta_{j}]\cdot F_{1j}+\Big[M_{d}\odot \varphi_{2,j}\Big]\odot F_{2j}\Bigg\}.
    \end{align}
    \item $B_{n,1}$ in equation~\eqref{eq:general_B_n_1}, we get
    \begin{align}
        \label{eq:general_limit_3}
        \frac{B_{n,1}-R_3(x)}{m_n\mathcal{D}_{2n}}\to\sum_{j:|\mathcal{A}_j|=1}[\phi^{\top}_{1,j}x\cdot H_j].
    \end{align}
    \item $B_{n,2}$ in equation~\eqref{eq:general_B_n_2}, we get
    \begin{align}
        \label{eq:general_limit_4}
        \frac{B_{n,2}-R_4(x)}{m_n\mathcal{D}_{2n}}\to\sum_{j:|\mathcal{A}_j|>1}\Big[\phi^{\top}_{1,j}x+x^{\top}\Big(M_{d}\odot \phi_{2,j}\Big)x\Big]\cdot H_j.
    \end{align}
    \item $E_n(x)$ in equation~\eqref{eq:general_Q_n}, we get
    \begin{align}
        \label{eq:general_limit_5}
        \frac{E_n(x)}{m_n\mathcal{D}_{2n}}\to\sum_{j=1}^{k_*}\xi_j[F_j-H_j].
    \end{align}
\end{itemize}
Due to the result in equation~\eqref{eq:general_zero_limit}, we deduce that the limits in equations~\eqref{eq:general_limit_1}, \eqref{eq:general_limit_2}, \eqref{eq:general_limit_3}, \eqref{eq:general_limit_4} and \eqref{eq:general_limit_5} sum up to zero. 

Now, we show that all the values of $\phi^{(u)}_{1,j}$, $\phi^{(uu)}_{2,j}$, $\varphi^{(u)}_{1,j}$, $\varphi^{(uu)}_{2,j}$ and $\xi_j$, for $j\in[k_*]$, are equal to zero. For that purpose, we first denote $J_1,J_2,\ldots,J_{\ell}$ as the partition of the set $\{\exp((\boj)^{\top}x):j\in[k_*]\}$ for some $\ell\leq k_*$ such that 
\begin{itemize}
    \item[(i)] $\beta^*_{1j}=\beta^*_{1j'}$ for any $j,j'\in J_i$ and $i\in[\ell]$;
    \item[(ii)] $\beta^*_{1j}\neq\beta^*_{1j'}$ when $j$ and $j'$ do not belong to the same set $J_i$ for any $i\in[\ell]$.
\end{itemize}
Then, the set $\{\exp((\beta^*_{1j_1})^{\top}x),\ldots,\exp((\beta^*_{1j_\ell})^{\top}x)\}$, where $j_i\in J_i$, is linearly independent. Since the limits in equations~\eqref{eq:general_limit_1}, \eqref{eq:general_limit_2}, \eqref{eq:general_limit_3}, \eqref{eq:general_limit_4} and \eqref{eq:general_limit_5} sum up to zero, we get for any $i\in[\ell]$ that
\begin{align*}
    &\sum_{j\in J_i:|\mathcal{A}_j|=1}\Big[(\xi_j+\phi^{\top}_{1,j}x)\cdot h_j+\varphi_{1,j}^{\top}h_{1j}\Big]+\sum_{j\in J_i:|\mathcal{A}_j|>1}\Bigg\{\Big[\phi^{\top}_{1,j}x+x^{\top}\Big(M_{d}\odot \phi_{2,j}\Big)x\Big]\cdot h_j\\
    &+[\varphi_{1,j}^{\top}+x^{\top}\zeta_{j}]\cdot h_{1j}+\Big[M_{d}\odot \varphi_{2,j}\Big]\odot h_{2j}\Bigg\}-\sum_{j\in J_i:|\mathcal{A}_j|=1}[(\phi^{\top}_{1,j}x+\xi_j)\cdot f_{G_*}(x)]\\
    &-\sum_{j\in J_i:|\mathcal{A}_j|>1}\Big[\xi_j+\phi^{\top}_{1,j}x+x^{\top}\Big(M_{d}\odot \phi_{2,j}\Big)x\Big]\cdot f_{G_*}(x)=0,
\end{align*}
where we denote $h_{j}:=h(x,\ej)$, $h_{1j}:=\frac{\partial h}{\partial \eta}(x,\ej)$ and $h_{2j}:=\frac{\partial^2 h}{\partial\eta\partial\eta^{\top}}(x,\ej)$. Recall that the expert function $h$ satisfies conditions in Definition~\ref{def:general_conditions}, then the following set is linearly independent
\begin{align*}
    \Big\{x^{\nu}\cdot\frac{\partial^{|\tau_1|+|\tau_2|}h}{\partial\eta^{\tau_1}\partial\eta^{\tau_2}}(x,\ej), \  x^{\nu}\cdot f_{G_*}(x):\nu\in\mathbb{N}^d, \ \tau_1,\tau_2\in\mathbb{N}^{q}, \ 0\leq |\nu|+|\tau_1|+|\tau_2|\leq 2, \ j\in[k_*]\Big\}.
\end{align*}
is linearly independent. Therefore, we obtain that $\xi_j=0$, $\phi_{1,j}=\varphi_{1,j}=\zerod$ and $\phi_{2,j}=\varphi_{2,j}=\zeta_{j}=\mathbf{0}_{d\times d}$ for any $j\in J_i$ and $i\in[\ell]$. In other words, those results hold true for any $j\in[k_*]$, which contradicts to the fact that at least one among $\phi^{(u)}_{1,j}$, $\phi^{(uu)}_{2,j}$, $\varphi^{(u)}_{1,j}$, $\varphi^{(uu)}_{2,j}$ and $\xi_j$, for $j\in[k_*]$, is different from zero. Thus, we achieve the inequality~\eqref{eq:general_local_inequality}, i.e.
\begin{align*}
    \lim_{\varepsilon\to0}\inf_{G\in\mathcal{G}_{k}(\Theta):\mathcal{D}_1(G,G_*)\leq\varepsilon}\normf{f_{G}-f_{G_*}}/\mathcal{D}_1(G,G_*)>0.
\end{align*}
As a consequence, there exists some $\varepsilon'>0$ such that
\begin{align*}
    \inf_{G\in\mathcal{G}_{k}(\Theta):\mathcal{D}_1(G,G_*)\leq\varepsilon'}\normf{f_{G}-f_{G_*}}/\mathcal{D}_1(G,G_*)>0.
\end{align*}
\textbf{Global part:} Given the above result, it suffices to demonstrate that 
\begin{align}
    \label{eq:general_global_inequality}
     \inf_{G\in\mathcal{G}_{k}(\Theta):\mathcal{D}_1(G,G_*)>\varepsilon'}\normf{f_{G}-f_{G_*}}/\mathcal{D}_1(G,G_*)>0.
\end{align}
Assume by contrary that the inequality~\eqref{eq:general_global_inequality} does not hold true, then we can find a sequence of mixing measures $G'_n\in\mathcal{G}_{k}(\Theta)$ such that $\mathcal{D}_1(G'_n,G_*)>\varepsilon'$ and
\begin{align*}
    \lim_{n\to\infty}\frac{\normf{f_{G'_n}-f_{G_*}}}{\mathcal{D}_1(G'_n,G_*)}=0,
\end{align*}
which indicates that $\normf{f_{G'_n}-f_{G_*}}\to0$ as $n\to\infty$. Recall that $\Theta$ is a compact set, therefore, we can replace the sequence $G'_n$ by one of its subsequences that converges to a mixing measure $G'\in\mathcal{G}_{k}(\Omega)$. Since $\mathcal{D}_1(G'_n,G_*)>\varepsilon'$, we deduce that $\mathcal{D}_1(G',G_*)>\varepsilon'$. 

Next, by invoking the Fatou's lemma, we have that
\begin{align*}
    0=\lim_{n\to\infty}\normf{f_{G'_n}-f_{G_*}}^2\geq \int\liminf_{n\to\infty}\Big|f_{G'_n}(x)-f_{G_*}(x)\Big|^2~\dint\mu(x).
\end{align*}
Thus, we get that $f_{G'}(x)=f_{G_*}(x)$ for almost every $x$. 
From Proposition~\ref{prop:general_identifiability}, we deduce that $G'\equiv G_*$. Consequently, it follows that $\mathcal{D}_1(G',G_*)=0$, contradicting the fact that $\mathcal{D}_1(G',G_*)>\varepsilon'>0$. 

Hence, the proof is completed.

\subsection{Proof of Theorem~\ref{theorem:activation_experts}}
\label{appendix:activation_experts}
In this proof, we focus on demonstrating the following inequality:
\begin{align}
    \label{eq:activation_universal_inequality}
    \inf_{G\in\mathcal{G}_{k}(\Theta)}\normf{f_{G}-f_{G_*}}/\mathcal{D}_2(G,G_*)>0.
\end{align}
To this end, we divide the proof of the above inequality into local and global parts in the sequel.

\textbf{Local part:} In this part, we show that
\begin{align}
    \label{eq:activation_local_inequality}
    \lim_{\varepsilon\to0}\inf_{G\in\mathcal{G}_{k}(\Theta):\mathcal{D}_2(G,G_*)\leq\varepsilon}\normf{f_{G}-f_{G_*}}/\mathcal{D}_2(G,G_*)>0.
\end{align}
Assume by contrary that the above inequality does not hold true, then there exists a sequence of mixing measures $G_n=\sum_{i=1}^{k_*}\exp(\beta^n_{0i})\delta_{(\beta^n_{1i},a^n_i,b^n_i)}$ in $\mathcal{G}_{k}(\Theta)$ such that $\mathcal{D}_{2n}:=\mathcal{D}_2(G_n,G_*)\to0$ and
\begin{align}
    \label{eq:activation_ratio_limit}
    \normf{f_{G_n}-f_{G_*}}/\mathcal{D}_{2n}\to0,
\end{align}
as $n\to\infty$. Let us denote by $\mathcal{A}^n_j:=\mathcal{A}_j(G_n)$ a Voronoi cell of $G_n$ generated by the $j$-th components of $G_*$. Since our arguments are assymptotic, we may assume that those Voronoi cells do not depend on the sample size, i.e. $\mathcal{A}_j=\mathcal{A}^n_j$. Thus, the Voronoi loss $\mathcal{D}_{2n}$ can be represented as
\begin{align}
    \label{eq:activation_loss_proof}
   \mathcal{D}_{2n}&:=\sum_{j=1}^{k_*}\Big|\sum_{i\in\mathcal{A}_j}\exp(\bzin)-\exp(\beta^*_{0j})\Big|+\sum_{j:|\mathcal{A}_j|>1}\sum_{i\in\mathcal{A}_j}\exp(\bzin)\Big[\|\dboijn\|^2+\|\daijn\|^2+|\dbijn|^2\Big]\nonumber\\
    &\hspace{6cm}+\sum_{j:|\mathcal{A}_j|=1}\sum_{i\in\mathcal{A}_j}\exp(\beta_{0i})\Big[\|\dboijn\|+\|\daijn\|+|\dbijn|\Big],
\end{align}
where we denote $\dboijn:=\beta^n_{1i}-\beta^*_{1j}$, $\daijn:=a^n_i-a^*_j$ and $\dbijn:=b^n_i-b^*_j$.

Since $\mathcal{D}_{2n}\to0$, we get that $(\boin,\ain,\bin)\to(\boj,\aj,\bj)$ and $\exp(\bzin)\to\exp(\bzj)$ as $n\to\infty$ for any $i\in\mathcal{A}_j$ and $j\in[k_*]$. Now, we divide the proof of local part into three steps as follows:

\textbf{Step 1.} In this step, we decompose the term $Q_n(x):=[\sum_{j=1}^{k_*}\exp((\beta^*_{1j})^{\top}x+\beta^*_{0j})]\cdot[f_{G_n}(x)-f_{G_*}(x)]$ into a combination of linearly independent elements using Taylor expansion. In particular, let us denote $F(x;\beta_1,a,b):=\exp(\beta_1^{\top}x)\sigma(a^{\top}x+b)$ and $H(x;\beta_1)=\exp(\beta_1^{\top}x)f_{G_n}(x)$, then we have
\begin{align}
    \label{eq:activation_Q_n}
    Q_n(x)&=\sum_{j=1}^{k_*}\sum_{i\in\mathcal{A}_j}\exp(\beta^n_{0i})\Big[F(x;\boin,\ain,\bin)-F(x;\boj,\aj,\bj)\Big]\nonumber\\
    &-\sum_{j=1}^{k_*}\sum_{i\in\mathcal{A}_j}\exp(\beta^n_{0i})\Big[H(x;\boin)-H(x;\boj)\Big]\nonumber\\
    &+\sum_{j=1}^{k_*}\Big(\sum_{i\in\mathcal{A}_j}\exp(\bzin)-\exp(\bzj)\Big)\Big[F(x;\boj,\aj,\bj)-H(x;\boj)\Big]\nonumber\\
    &:=A_n(x)-B_n(x)+E_n(x).
\end{align}
\textbf{Decomposition of $A_n(x)$.} Next, we continue to separate the term $A_n(x)$ into two parts as follows:
\begin{align*}
    A_n(x)&:=\sum_{j:|\mathcal{A}_j|=1}\sum_{i\in\mathcal{A}_j}\exp(\beta^n_{0i})\Big[F(x;\boin,\ain,\bin)-F(x;\boj,\aj,\bj)\Big]\\
    &+\sum_{j:|\mathcal{A}_j|>1}\sum_{i\in\mathcal{A}_j}\exp(\beta^n_{0i})\Big[F(x;\boin,\ain,\bin)-F(x;\boj,\aj,\bj)\Big]\\
    &:=A_{n,1}+A_{n,2}.
\end{align*}
By means of the first-order Taylor expansion, we have
\begin{align*}
    A_{n,1}=\sum_{j:|\mathcal{A}_j|=1}\sum_{i\in\mathcal{A}_j}\exp(\beta^n_{0i})\sum_{|\alpha|=1}(\dboijn)^{\alpha_1}(\daijn)^{\alpha_2}(\dbijn)^{\alpha_3}\cdot\frac{\partial F}{\partial\beta_1^{\alpha_1}\partial a^{\alpha_2}\partial b^{\alpha_3}}(x;\boj,\aj,\bj)+R_1(x),
\end{align*}
where $R_1(x)$ is a Taylor remainder such that $R_1(x)/\mathcal{D}_{2n}\to0$ as $n\to\infty$. By taking the first derivatives of $F$ w.r.t its parameters, we get
\begin{align*}
    \frac{\partial F}{\partial\beta_1}(x;\boj,\aj,\bj)&=x\exp((\boj)^{\top}x)\cdot\sigma((\aj)^{\top}x+\bj)=x\cdot F(x;\boj,\aj,\bj),\\
    \frac{\partial F}{\partial a}(x;\boj,\aj,\bj)&=x\exp((\boj)^{\top}x)\cdot\sigma^{(1)}((\aj)^{\top}x+\bj)=x\cdot F_1(x;\boj,\aj,\bj),\\
    \frac{\partial F}{\partial b}(x;\boj,\aj,\bj)&=\exp((\boj)^{\top}x)\cdot\sigma^{(1)}((\aj)^{\top}x+\bj)=F_1(x;\boj,\aj,\bj),
\end{align*}
where we denote $F_{\tau}(x;\boj,\aj,\bj)=\exp((\boj)^{\top}x)\cdot\sigma^{(\tau)}((\aj)^{\top}x+\bj)$. Thus, we can rewrite $A_{n,1}$ as
\begin{align}
    \label{eq:activation_A_n_1}
    A_{n,1}&=\sum_{j:|\mathcal{A}_j|=1}C_{n,1,j}(x)+R_1(x),
\end{align}
where
\begin{align*}
    C_{n,1,j}(x)=\sum_{i\in\mathcal{A}_j}\exp(\beta^n_{0i})\Big[(\dboijn)^{\top}x\cdot F(x;\boj,\aj,\bj)+((\daijn)^{\top}x+(\dbijn))\cdot F_1(x;\boj,\aj,\bj)\Big].
\end{align*}

Next, by applying the second-order Taylor expansion, $A_{n,2}$ can be represented as
\begin{align*}
    A_{n,2}=\sum_{j:|\mathcal{A}_j|>1}\sum_{i\in\mathcal{A}_j}\exp(\beta^n_{0i})\sum_{|\alpha|=1}^{2}\frac{1}{\alpha!}(\dboijn)^{\alpha_1}(\daijn)^{\alpha_2}(\dbijn)^{\alpha_3}\cdot\frac{\partial^{|\alpha_1|+|\alpha_2|+\alpha_3} F}{\partial\beta_1^{\alpha_1}\partial a^{\alpha_2}\partial b^{\alpha_3}}(x;\boj,\aj,\bj)+R_2(x),
\end{align*}
where $R_2(x)$ is a Taylor remainder such that $R_2(x)/\mathcal{D}_{2n}\to0$ as $n\to\infty$. The second derivatives of $F$ w.r.t its parameters are given by
\begin{align*}
    \frac{\partial^2F}{\partial\beta\partial\beta^{\top}}(x;\boj,\aj,\bj)=xx^{\top}\cdot F(x;\boj,\aj,\bj)&,\quad \frac{\partial^2F}{\partial\beta_1\partial a^{\top}}=xx^{\top}\cdot F_1(x;\boj,\aj,\bj),\\
    \frac{\partial^2F}{\partial\beta_1\partial b}=x\cdot F_1(x;\boj,\aj,\bj)&, \quad \frac{\partial^2F}{\partial a\partial a^{\top}}=xx^{\top}\cdot F_2(x,\boj,\aj,\bj),\\
    \frac{\partial^2F}{\partial a\partial b}=x\cdot F_2(x,\boj,\aj,\bj)&, \quad \frac{\partial^2F}{\partial b^2}=F_2(x,\boj,\aj,\bj).
\end{align*}
Therefore, the term $A_{n,2}$ becomes
\begin{align}
     \label{eq:activation_A_n_2}
    A_{n,2}=\sum_{j:|\mathcal{A}_j|>1}[C_{n,1,j}(x)+C_{n,2,j}(x)]+R_2(x),
\end{align}
where 
\begin{align*}
    &C_{n,2,j}(x):=\sum_{i\in\mathcal{A}_j}\exp(\bzin)\Bigg\{\Big[x^{\top}\Big(M_{d}\odot (\dboijn)(\dboijn)^{\top}\Big)x\Big]\cdot F(x;\boj,\aj,\bj)\\
    &+\Big[x^{\top}\Big(M_{d}\odot (\daijn)(\daijn)^{\top}\Big)x+(\dbijn)(\dboijn)^{\top}x+x^{\top}(\dboijn)(\daijn)^{\top}x\Big]\cdot F_1(x;\boj,\aj,\bj)\\
    &\hspace{7cm}+\Big[\frac{1}{2}(\dbijn)^2+(\dbijn)(\daijn)^{\top}x\Big]\cdot F_2(x;\boj,\aj,\bj)\Bigg\}.
\end{align*}
\textbf{Decomposition of $B_n(x)$.} Subsequently, we also divide $B_n(x)$ into two terms based on the Voronoi cells as
\begin{align*}
    B_n(x)&=\sum_{j:|\mathcal{A}_j|=1}\sum_{i\in\mathcal{A}_j}\exp(\beta^n_{0i})\Big[H(x;\boin)-H(x;\boj)\Big]\\
    &+\sum_{j:|\mathcal{A}_j|>1}\sum_{i\in\mathcal{A}_j}\exp(\beta^n_{0i})\Big[H(x;\boin)-H(x;\boj)\Big]\\
    &:=B_{n,1}+B_{n,2}.
\end{align*}
By means of the first-order Taylor expansion, we have
\begin{align}
    \label{eq:activation_B_n_1}
    B_{n,1}=\sum_{j:|\mathcal{A}_j|=1}\sum_{i\in\mathcal{A}_j}\exp(\bzin)(\dboijn)^{\top}x\cdot H(x;\boj)+R_3(x),
\end{align}
where $R_3(x)$ is a Taylor remainder such that $R_3(x)/\mathcal{D}_{2n}\to0$ as $n\to\infty$. Meanwhile, by applying the second-order Taylor expansion, we get
\begin{align}
      \label{eq:activation_B_n_2}
    B_{n,2}=\sum_{j:|\mathcal{A}_j|>1}\sum_{i\in\mathcal{A}_j}\exp(\bzin)\Big[(\dboijn)^{\top}x+x^{\top}\Big(M_{d}\odot(\dboijn)(\dboijn)^{\top}\Big)x\Big]\cdot H(x;\boj) + R_4(x),
\end{align}
where $R_4(x)$ is a Taylor remainder such that $R_4(x)/\mathcal{D}_{2n}\to0$ as $n\to\infty$.

Putting the above results together, we see that $[A_n(x)-R_1(x)-R_2(x)]/\mathcal{D}_{2n}$, $[B_n(x)-R_3(x)-R_4(x)]/\mathcal{D}_{2n}$ and $E_n(x)/\mathcal{D}_{2n}$ can be written as a combination of elements from set $\mathcal{S}:=\cup_{\tau=0}^{3}\mathcal{S}_{\tau}$ in which
\begin{align*}
    \mathcal{S}_{0}&:=\Big\{F(x;\boj,\aj,\bj), \ x^{(u)}F(x;\boj,\aj,\bj), \ x^{(u)}x^{(v)}F(x;\boj,\aj,\bj):u,v\in[d], \ j\in[k_*]\Big\},\\
    \mathcal{S}_{1}&:=\Big\{F_1(x;\boj,\aj,\bj), \ x^{(u)}F_1(x;\boj,\aj,\bj), \ x^{(u)}x^{(v)}F_1(x;\boj,\aj,\bj):u,v\in[d], \ j\in[k_*]\Big\},\\
    \mathcal{S}_{2}&:=\Big\{F_2(x;\boj,\aj,\bj), \ x^{(u)}F_2(x;\boj,\aj,\bj), \ x^{(u)}x^{(v)}F_2(x;\boj,\aj,\bj):u,v\in[d], \ j\in[k_*]\Big\},\\
    \mathcal{S}_{3}&:=\Big\{H(x;\boj), \ x^{(u)}H(x;\boj), \ x^{(u)}x^{(v)}H(x;\boj):u,v\in[d], \ j\in[k_*]\Big\}.
\end{align*}
\textbf{Step 2.} In this step, we prove by contradiction that at least one among coefficients in the representations of $[A_n(x)-R_1(x)-R_2(x)]/\mathcal{D}_{2n}$, $[B_n(x)-R_3(x)-R_4(x)]/\mathcal{D}_{2n}$ and $E_n(x)/\mathcal{D}_{2n}$ does not go to zero as $n$ tends to infinity. Indeed, assume that all of them converge to zero. Then, by considering the coefficients of 
\begin{itemize}
    \item $F(x;\boj,\aj,\bj)$ for $j\in[k_*]$, we get that $\frac{1}{\mathcal{D}_{2n}}\cdot\sum_{j=1}^{k_*}\Big|\sum_{i\in\mathcal{A}_j}\exp(\bzin)-\exp(\bzj)\Big|\to0$;
    \item $x^{(u)}F(x;\boj,\aj,\bj)$ for $u\in[d]$ and $j:|\mathcal{A}_j|=1$, we get that $\frac{1}{\mathcal{D}_{2n}}\cdot\sum_{j:|\mathcal{A}_j|=1}\sum_{i\in\mathcal{A}_j}\exp(\bzin)\|\dboijn\|_1\to0$;
    \item $x^{(u)}F_1(x;\boj,\aj,\bj)$ for $u\in[d]$ and $j:|\mathcal{A}_j|=1$, we get that $\frac{1}{\mathcal{D}_{2n}}\cdot\sum_{j:|\mathcal{A}_j|=1}\sum_{i\in\mathcal{A}_j}\exp(\bzin)\|\daijn\|_1\to0$;
    \item $F_1(x;\boj,\aj,\bj)$ for $j:|\mathcal{A}_j|=1$, we get that $\frac{1}{\mathcal{D}_{2n}}\cdot\sum_{j:|\mathcal{A}_j|=1}\sum_{i\in\mathcal{A}_j}\exp(\bzin)|\dbijn|_1\to0$;
    \item $[x^{(u)}]^2F(x;\boj,\aj,\bj)$ for $u\in[d]$ and $j:|\mathcal{A}_j|>1$, we get that $\frac{1}{\mathcal{D}_{2n}}\cdot\sum_{j:|\mathcal{A}_j|>1}\sum_{i\in\mathcal{A}_j}\exp(\bzin)\|\dboijn\|^2\to0$;
    \item $[x^{(u)}]^2F_2(x;\boj,\aj,\bj)$ for $u\in[d]$ and $j:|\mathcal{A}_j|>1$, we get that $\frac{1}{\mathcal{D}_{2n}}\cdot\sum_{j:|\mathcal{A}_j|>1}\sum_{i\in\mathcal{A}_j}\exp(\bzin)\|\daijn\|^2\to0$;
    \item $F_2(x;\boj,\aj,\bj)$ for $j:|\mathcal{A}_j|>1$, we get that $\frac{1}{\mathcal{D}_{2n}}\cdot\sum_{j:|\mathcal{A}_j|>1}\sum_{i\in\mathcal{A}_j}\exp(\bzin)\|\dbijn\|^2\to0$.
\end{itemize}
By taking the summation of the above limits, we obtain that $1=\mathcal{D}_{2n}/\mathcal{D}_{2n}\to0$ as $n\to\infty$, which is a contradiction. Therefore, not all the coefficients in the representations of $[A_n(x)-R_1(x)-R_2(x)]/\mathcal{D}_{2n}$, $[B_n(x)-R_3(x)-R_4(x)]/\mathcal{D}_{2n}$ and $E_n(x)/\mathcal{D}_{2n}$ go to zero.

\textbf{Step 3.} In this step, we point out a contradiction following from the result in Step 2. Let us denote by $m_n$ the maximum of the absolute values of the coefficients in the representations of $[A_n(x)-R_1(x)-R_2(x)]/\mathcal{D}_{2n}$, $[B_n(x)-R_3(x)-R_4(x)]/\mathcal{D}_{2n}$ and $E_n(x)/\mathcal{D}_{2n}$. Since at least one among those coefficients does not approach zero, we obtain that $1/m_n\not\to\infty$. 

Recall the hypothesis in equation~\eqref{eq:activation_ratio_limit} that $\normf{f_{G_n}-f_{G_*}}/\mathcal{D}_{2n}\to0$ as $n\to\infty$, which indicates that $\|f_{G_n}-f_{G_*}\|_{L^1(\mu)}/\mathcal{D}_{2n}\to0$. By means of the Fatou's lemma, we have
\begin{align*}
    0=\lim_{n\to\infty}\frac{\|f_{G_n}-f_{G_*}\|_{L^1(\mu)}}{m_n\mathcal{D}_{2n}}\geq \int \liminf_{n\to\infty}\frac{|f_{G_n}(x)-f_{G_*}(x)|}{m_n\mathcal{D}_{2n}}\dint\mu(x)\geq 0.
\end{align*}
This result implies that $[f_{G_n}(x)-f_{G_*}(x)]/[m_n\mathcal{D}_{2n}]$ for almost every $x$. Since the term $\sum_{j=1}^{k_*}\exp((\beta^*_{1j})^{\top}x+\beta^*_{0j})$ is bounded, we deduce that $Q_n(x)/[m_n\mathcal{D}_{2n}]\to0$, or equivalently, 
\begin{align}
    \label{eq:activation_zero_limit}
    \lim_{n\to\infty}\frac{1}{m_n\mathcal{D}_{2n}}\cdot\Big[(A_{n,1}-R_1(x)+A_{n,2}-R_2(x))-(B_{n,1}-R_3(x)+B_{n,2}-R_4(x))+E_n(x)\Big]\to0.
\end{align}
Let us denote
\begin{align*}
    \frac{1}{m_n\mathcal{D}_{2n}}\cdot\sum_{i\in\mathcal{A}_j}\exp(\bzin)(\dboijn)\to\phi_{1,j}&,\quad  \frac{1}{m_n\mathcal{D}_{2n}}\cdot\sum_{i\in\mathcal{A}_j}\exp(\bzin)(\dboijn)(\dboijn)^{\top}\to\phi_{2,j},\\
    \frac{1}{m_n\mathcal{D}_{2n}}\cdot\sum_{i\in\mathcal{A}_j}\exp(\bzin)(\daijn)\to\varphi_{1,j}&,\quad \frac{1}{m_n\mathcal{D}_{2n}}\cdot\sum_{i\in\mathcal{A}_j}\exp(\bzin)(\daijn)(\daijn)^{\top}\to\varphi_{2,j},\\
    \frac{1}{m_n\mathcal{D}_{2n}}\cdot\sum_{i\in\mathcal{A}_j}\exp(\bzin)(\dbijn)\to\kappa_{1,j}&,\quad \frac{1}{m_n\mathcal{D}_{2n}}\cdot\sum_{i\in\mathcal{A}_j}\exp(\bzin)(\dbijn)^2\to\kappa_{2,j},\\
     \frac{1}{m_n\mathcal{D}_{2n}}\cdot\sum_{i\in\mathcal{A}_j}\exp(\bzin)(\dboijn)(\daijn)^{\top}\to\zeta_{1,j}&,\quad \frac{1}{m_n\mathcal{D}_{2n}}\cdot\sum_{i\in\mathcal{A}_j}\exp(\bzin)(\dbijn)(\dboijn)\to\zeta_{2,j},\\
     \frac{1}{m_n\mathcal{D}_{2n}}\cdot\sum_{i\in\mathcal{A}_j}\exp(\bzin)(\dbijn)(\daijn)\to\zeta_{3,j}&,\quad \frac{1}{m_n\mathcal{D}_{2n}}\cdot\Big(\sum_{i\in\mathcal{A}_j}\exp(\bzin)-\exp(\bzj)\Big)\to\xi_j&.
\end{align*}
Here, at least one among $\phi^{(u)}_{1,j}$, $\phi^{(uu)}_{2,j}$, $\varphi^{(u)}_{1,j}$, $\varphi^{(uu)}_{2,j}$, $\kappa_{1,j}$, $\kappa_{2,j}$ and $\xi_j$, for $j\in[k_*]$, is different from zero, which results from Step 2. Additionally, let us denote $F_{\tau j}:=F_{\tau}(x;\boj,\aj,\bj)$ and $H_j=H(x;\boj)$ for short, then from the formulation of 
\begin{itemize}
    \item $A_{n,1}$ in equation~\eqref{eq:activation_A_n_1}, we get 
    \begin{align}
        \label{eq:activation_limit_1}
        \frac{A_{n,1}-R_1(x)}{m_n\mathcal{D}_{2n}}\to\sum_{j:|\mathcal{A}_j|=1}\Big[\phi^{\top}_{1,j}x\cdot F_j+(\kappa_{1,j}+\varphi_{1,j}^{\top}x)\cdot F_{1j}\Big].
    \end{align}
    \item $A_{n,2}$ in equation~\eqref{eq:activation_A_n_2}, we get
    \begin{align}
        &\frac{A_{n,2}-R_2(x)}{m_n\mathcal{D}_{2n}}\to\sum_{j:|\mathcal{A}_j|>1}\Bigg\{\Big[\phi^{\top}_{1,j}x+x^{\top}\Big(M_{d}\odot \phi_{2,j}\Big)x\Big]\cdot F_j+[\kappa_{1,j}+(\varphi_{1,j}+\zeta_{2,j})^{\top}x+x^{\top}\zeta_{1,j}x]\cdot F_{1j}\nonumber\\
        \label{eq:activation_limit_2}
        &\hspace{8cm}+\Big[\frac{1}{2}\kappa_{2,j}+\zeta^{\top}_{3,j}x+x^{\top}\Big(M_{d}\odot \varphi_{2,j}\Big)x\Big]\cdot F_{2j}
    \end{align}
    \item $B_{n,1}$ in equation~\eqref{eq:activation_B_n_1}, we get
    \begin{align}
        \label{eq:activation_limit_3}
        \frac{B_{n,1}-R_3(x)}{m_n\mathcal{D}_{2n}}\to\sum_{j:|\mathcal{A}_j|=1}[\phi^{\top}_{1,j}x\cdot H_j].
    \end{align}
    \item $B_{n,2}$ in equation~\eqref{eq:activation_B_n_2}, we get
    \begin{align}
        \label{eq:activation_limit_4}
        \frac{B_{n,2}-R_4(x)}{m_n\mathcal{D}_{2n}}\to\sum_{j:|\mathcal{A}_j|>1}\Big[\phi^{\top}_{1,j}x+x^{\top}\Big(M_{d}\odot \phi_{2,j}\Big)x\Big]\cdot H_j.
    \end{align}
    \item $E_n(x)$ in equation~\eqref{eq:activation_Q_n}, we get
    \begin{align}
        \label{eq:activation_limit_5}
        \frac{E_n(x)}{m_n\mathcal{D}_{2n}}\to\sum_{j=1}^{k_*}\xi_j[F_j-H_j].
    \end{align}
\end{itemize}
Due to the result in equation~\eqref{eq:activation_zero_limit}, we deduce that the limits in equations~\eqref{eq:activation_limit_1}, \eqref{eq:activation_limit_2}, \eqref{eq:activation_limit_3}, \eqref{eq:activation_limit_4} and \eqref{eq:activation_limit_5} sum up to zero. 

Now, we show that all the values of $\phi^{(u)}_{1,j}$, $\phi^{(uu)}_{2,j}$, $\varphi^{(u)}_{1,j}$, $\varphi^{(uu)}_{2,j}$, $\kappa_{1,j}$, $\kappa_{2,j}$ and $\xi_j$, for $j\in[k_*]$, are equal to zero. For that purpose, we first denote $J_1,J_2,\ldots,J_{\ell}$ as the partition of the set $\{\exp((\boj)^{\top}x):j\in[k_*]\}$ for some $\ell\leq k_*$ such that 
\begin{itemize}
    \item[(i)] $\beta^*_{0j}=\beta^*_{0j'}$ for any $j,j'\in J_i$ and $i\in[\ell]$;
    \item[(ii)] $\beta^*_{0j}\neq\beta^*_{0j'}$ when $j$ and $j'$ do not belong to the same set $J_i$ for any $i\in[\ell]$.
\end{itemize}
Then, the set $\{\exp((\beta^*_{0j_1})^{\top}x),\ldots,\exp((\beta^*_{0j_\ell})^{\top}x)\}$, where $j_i\in J_i$, is linearly independent. Since the limits in equations~\eqref{eq:activation_limit_1}, \eqref{eq:activation_limit_2}, \eqref{eq:activation_limit_3}, \eqref{eq:activation_limit_4} and \eqref{eq:activation_limit_5} sum up to zero, we get for any $i\in[\ell]$ that
\begin{align*}
    &\sum_{j\in J_i:|\mathcal{A}_j|=1}\Big[(\phi^{\top}_{1,j}x+\xi_j)\cdot \sigma_j+(\kappa_{1,j}+\varphi_{1,j}^{\top}x)\cdot \sigma^{(1)}_j\Big]+\sum_{j\in J_i:|\mathcal{A}_j|>1}\Bigg\{\Big[\xi_j+\phi^{\top}_{1,j}x+x^{\top}\Big(M_{d}\odot \phi_{2,j}\Big)x\Big]\cdot \sigma_j\\
    &+[\kappa_{1,j}+(\varphi_{1,j}+\zeta_{2,j})^{\top}x+x^{\top}\zeta_{1,j}x]\cdot \sigma^{(1)}_j+\Big[\frac{1}{2}\kappa_{2,j}+\zeta^{\top}_{3,j}x+x^{\top}\Big(M_{d}\odot \varphi_{2,j}\Big)x\Big]\cdot \sigma^{(2)}_j\Bigg\}\\
    &-\sum_{j\in J_i:|\mathcal{A}_j|=1}[(\phi^{\top}_{1,j}x+\xi_j)\cdot f_{G_*}(x)]-\sum_{j\in J_i:|\mathcal{A}_j|>1}\Big[\xi_j+\phi^{\top}_{1,j}x+x^{\top}\Big(M_{d}\odot \phi_{2,j}\Big)x\Big]\cdot f_{G_*}(x)=0,
\end{align*}
where we denote $\sigma^{(\tau)}_j:=\sigma^{(\tau)}((\aj)^{\top}x+\bj)$.
Additionally, as $(a^*_1,b^*_1),\ldots,(a^*_{k_*},b^*_{k_*})$ are pairwise distinct, the experts $(a^*_1)^{\top}x+b^*_1,\ldots,(a^*_{k_*})^{\top}x+b^*_{k_*}$ are also pairwise distinct. Recall that the function $\sigma$ satisfies conditions in Definition~\ref{def:activation_conditions}, then the following set is linearly independent
\begin{align*}
    \Big\{x^{\nu}\sigma^{(\tau)}_j, x^{\nu}f_{G_*}(x):\nu\in\mathbb{N}^d,\tau\in\mathbb{N}, \ 0\leq |\nu|,\tau\leq 2, \ j\in[k_*]\Big\}.
\end{align*}
is linearly independent. Therefore, we obtain that $\kappa_{1,j}=\kappa_{2,j}=\xi_j=0$, $\phi_{1,j}=\varphi_{1,j}=\zeta_{2,j}=\zeta_{3,j}=\zerod$ and $\phi_{2,j}=\varphi_{2,j}=\zeta_{1,j}=\mathbf{0}_{d\times d}$ for any $j\in J_i$ and $i\in[\ell]$. In other words, those results hold true for any $j\in[k_*]$, which contradicts to the fact that at least one among $\phi^{(u)}_{1,j}$, $\phi^{(uu)}_{2,j}$, $\varphi^{(u)}_{1,j}$, $\varphi^{(uu)}_{2,j}$, $\kappa_{1,j}$, $\kappa_{2,j}$ and $\xi_j$, for $j\in[k_*]$, is different from zero. Thus, we achieve the inequality~\eqref{eq:activation_local_inequality}, i.e.
\begin{align*}
    \lim_{\varepsilon\to0}\inf_{G\in\mathcal{G}_{k}(\Theta):\mathcal{D}_2(G,G_*)\leq\varepsilon}\normf{f_{G}-f_{G_*}}/\mathcal{D}_2(G,G_*)>0.
\end{align*}
As a consequence, there exists some $\varepsilon'>0$ such that
\begin{align*}
    \inf_{G\in\mathcal{G}_{k}(\Theta):\mathcal{D}_2(G,G_*)\leq\varepsilon'}\normf{f_{G}-f_{G_*}}/\mathcal{D}_2(G,G_*)>0.
\end{align*}
\textbf{Global part:} Given the above result, it suffices to demonstrate that 
\begin{align}
    \label{eq:activation_global_inequality}
     \inf_{G\in\mathcal{G}_{k}(\Theta):\mathcal{D}_2(G,G_*)>\varepsilon'}\normf{f_{G}-f_{G_*}}/\mathcal{D}_2(G,G_*)>0.
\end{align}
Assume by contrary that the inequality~\eqref{eq:activation_global_inequality} does not hold true, then we can find a sequence of mixing measures $G'_n\in\mathcal{G}_{k}(\Theta)$ such that $\mathcal{D}_2(G'_n,G_*)>\varepsilon'$ and
\begin{align*}
    \lim_{n\to\infty}\frac{\normf{f_{G'_n}-f_{G_*}}}{\mathcal{D}_2(G'_n,G_*)}=0,
\end{align*}
which indicates that $\normf{f_{G'_n}-f_{G_*}}\to0$ as $n\to\infty$. Recall that $\Theta$ is a compact set, therefore, we can replace the sequence $G'_n$ by one of its subsequences that converges to a mixing measure $G'\in\mathcal{G}_{k}(\Omega)$. Since $\mathcal{D}_2(G'_n,G_*)>\varepsilon'$, we deduce that $\mathcal{D}_2(G',G_*)>\varepsilon'$. 

Next, by invoking the Fatou's lemma, we have that
\begin{align*}
    0=\lim_{n\to\infty}\normf{f_{G'_n}-f_{G_*}}^2\geq \int\liminf_{n\to\infty}\Big|f_{G'_n}(x)-f_{G_*}(x)\Big|^2~\dint\mu(x).
\end{align*}
Thus, we get that $f_{G'}(x)=f_{G_*}(x)$ for almost every $x$. 
From Proposition~\ref{prop:general_identifiability}, we deduce that $G'\equiv G_*$. Consequently, it follows that $\mathcal{D}_2(G',G_*)=0$, contradicting the fact that $\mathcal{D}_2(G',G_*)>\varepsilon'>0$. 

Hence, the proof is completed.

\subsection{Proof of Proposition~\ref{prop:activation_limit}}
\label{appendix:activation_limit}

It is sufficient to show that the following limit holds true for any $r\geq 1$:
\begin{align}
    \label{eq:activation_ratio_zero_limit}
    \lim_{\varepsilon\to0}\inf_{G\in\mathcal{G}_{k}(\Theta):\mathcal{D}_{3,r}(G,G_*)\leq\varepsilon}\frac{\normf{f_{G}-f_{G_*}}}{\mathcal{D}_{3,r}(G,G_*)}=0.
\end{align}
To this end, we need to construct a sequence of mixing measures $(G_n)$ that satisfies $\mathcal{D}_{3,r}(G_n,G_*)\to 0$ and
\begin{align*}
    \frac{\normf{f_{G_n}-f_{G_*}}}{\mathcal{D}_{3,r}(G_n,G_*)}\to0,
\end{align*}
as $n\to\infty$. Recall that under the Regime 2, at least one among parameters $a^*_1,\ldots,a^*_{k_*}$ is equal to $\zerod$. Without loss of generality, we may assume that $a^*_1=\zerod$. Next, let us take into account the sequence $G_n=\sum_{i=1}^{k_*+1}\exp(\beta^n_{0i})\delta_{(\beta^n_{1i},a^n_i,b^n_i)}$ in which
\begin{itemize}
    \item $\exp(\beta^n_{01})=\exp(\beta^n_{02})=\frac{1}{2}\exp(\beta^*_{01})$ and  $\exp(\beta^n_{0i})=\exp(\beta^*_{0(i-1)})$ for any $3\leq i\leq k_*+1$;
    \item $\beta^n_{11}=\beta^n_{12}=\beta^*_{11}$ and  $\beta^n_{1i}=\beta^*_{1(i-1)}$ for any $3\leq i\leq k_*+1$;
    \item $a^n_1=a^n_2=a^*_1=\zerod$ and $a^n_i=a^*_{i-1}$ for any $3\leq i\leq k_*+1$;
    \item $b^n_1=b^*_1+\frac{c}{n}$, $b^n_2=b^*_1+\frac{2c}{n}$ and  $b^n_{i}=b^*_{i-1}$ for any $3\leq i\leq k_*+1$,
\end{itemize}
where $c\in\mathbb{R}$ will be chosen later. Consequently, we get that
\begin{align*}
    \mathcal{D}_{3,r}(G_n,G_*)=\frac{1}{2}\exp(\beta^*_{01})\Big[\frac{c^r}{n^r}+\frac{(2c)^r}{n^r}\Big]=\mathcal{O}(n^{-r}).
\end{align*}
Next, we demonstrate that $\normf{f_{G_n}-f_{G_*}}/\mathcal{D}_{3,r}(G_n,G_*)\to0$. To this end, consider the quantity $Q_n(x):=[\sum_{j=1}^{k_*}\exp((\beta^*_{1j})^{\top}x+\beta^*_{0j})]\cdot[f_{G_n}(x)-f_{G_*}(x)]$, and decompose it as follows:
\begin{align*}
    Q_n(x)&=\sum_{j=1}^{k_*}\sum_{i\in\mathcal{A}_j}\exp(\beta^n_{0i})\Big[\exp((\boin)^{\top}x)\sigma((\ain)^{\top}x+\bin)-\exp((\boj)^{\top}x)\sigma((\aj)^{\top}x+\bj)\Big]\\
    &-\sum_{j=1}^{k_*}\sum_{i\in\mathcal{A}_j}\exp(\beta^n_{0i})\Big[\exp((\boin)^{\top}x)f_{G_n}(x)-\exp((\boj)^{\top}x)f_{G_n}(x)\Big]\\
    &+\sum_{j=1}^{k_*}\Big(\sum_{i\in\mathcal{A}_j}\exp(\bzin)-\exp(\bzj)\Big)\Big[\exp((\boj)^{\top}x)\sigma((\aj)^{\top}x+\bj)-\exp((\boj)^{\top}x)f_{G_n}(x)\Big]\\
    &:=A_n(x)-B_n(x)+E_n(x).
\end{align*}
From the definitions of $\beta^n_{1i},a^n_i$ and $b^n_i$, we can verify that $B_n(x)=E_n(x)=0$. Additionally, we can represent $A_n(x)$ as
\begin{align*}
    A_n(x)&=\sum_{i=1}^{2}\exp(\beta^*_{01})\exp((\beta^*_{11})^{\top}x)\Big[\sigma(\bin)-\sigma(b^*_1)\Big].
\end{align*}
\textbf{When $r$ is odd:} By applying the Taylor expansion up to order $r$-th, we get that 
\begin{align*}
    A_n(x)&=\sum_{i=1}^{2}\exp(\beta^*_{01})\exp((\beta^*_{11})^{\top}x)\sum_{\alpha=1}^{r}\frac{(b^n_i-b^*_1)^{\alpha}}{\alpha!}\cdot \sigma^{(\alpha)}(b^*_1)+ R_1(x)\\
    &=\Big[\sum_{\alpha=1}^{r}\frac{(1+2^{\alpha})\sigma^{(\alpha)}(b^*_1)}{\alpha!n^r}\cdot c^{\alpha}\Big]\exp((\beta^*_{11})^{\top}x+\beta^*_{01}) + R_1(x),
\end{align*}
where $R_1(x)$ is a Taylor remainder such that $R_1(x)/\mathcal{D}_{3,r}(G_n,G_*)\to0$. Note that $\Big[\sum_{\alpha=1}^{r}\frac{(1+2^{\alpha})\sigma^{(\alpha)}(b^*_1)}{\alpha!n^{\alpha}}\cdot c^{\alpha}\Big]$ is an odd-order polynomial of $c$. Thus, we can choose $c$ as a root of this polynomial, which leads to the fact that $A_n(x)=0$. From the above results, we deduce that $Q_n(x)/\mathcal{D}_{3,r}(G_n,G_*)\to0$, or equivalently, $[f_{G_n}(x)-f_{G_*}(x)]/\mathcal{D}_{3,r}(G_n,G_*)\to0$ as $n\to\infty$ for almost every $x$. As a consequence, we achieve that $\normf{f_{G_n}-f_{G_*}}/\mathcal{D}_{3,r}(G_n,G_*)\to0$.

\textbf{When $r$ is even:} By means of the Taylor expansion of order $(r+1)$-th, we have
\begin{align*}
    A_n(x)&=\sum_{i=1}^{2}\exp(\beta^*_{01})\exp((\beta^*_{11})^{\top}x)\sum_{\alpha=1}^{r+1}\frac{(b^n_i-b^*_1)^{\alpha}}{\alpha!}\cdot \sigma^{(\alpha)}(b^*_1)+ R_2(x)\\
    &=\Big[\sum_{\alpha=1}^{r+1}\frac{(1+2^{\alpha})\sigma^{(\alpha)}(b^*_1)}{\alpha!n^r}\cdot c^{\alpha}\Big]\exp((\beta^*_{11})^{\top}x+\beta^*_{01}) + R_2(x),
\end{align*}
where $R_2(x)$ is a Taylor remainder such that $R_2(x)/\mathcal{D}_{3,r}(G_n,G_*)\to0$. Since $\Big[\sum_{\alpha=1}^{r+1}\frac{(1+2^{\alpha})\sigma^{(\alpha)}(b^*_1)}{\alpha!n^{\alpha}}\cdot c^{\alpha}\Big]$ is an odd-degree polynomial of variable $c$, we can argue in a similar fashion to the scenario when $r$ is odd to obtain that $\normf{f_{G_n}-f_{G_*}}/\mathcal{D}_{3,r}(G_n,G_*)\to0$.

Combine results from the above two cases of $r$, we reach the conclusion of claim~\eqref{eq:activation_ratio_zero_limit}.

\subsection{Proof of Theorem~\ref{theorem:singular_activation_experts}}
\label{appendix:singular_activation_experts}
Based on the result of Proposition~\ref{prop:activation_limit}, we demonstrate that the following minimax lower bound holds true for any $r\geq 1$:
\begin{align}
    \label{eq:minimax_activation_experts}
    \inf_{\overline{G}_n\in\mathcal{G}_{k}(\Theta)}\sup_{G\in\mathcal{G}_{k}(\Theta)\setminus\mathcal{G}_{k_*-1}(\Theta)}\bbE_{f_{G}}[\mathcal{D}_{3,r}(\overline{G}_n,G)]\gtrsim n^{-1/2}.
\end{align}
Indeed, from the Gaussian assumption on the noise variables, we obtain that $Y_{i}|X_{i} \sim \mathcal{N}(f_{G_{*}}(x_{i}), \sigma^2)$ for all $i \in [n]$. Now, from Proposition~\ref{prop:activation_limit}, for sufficiently small $\varepsilon>0$ and a fixed constant $C_1>0$ that we will choose later, we can find a mixing measure $G'_* \in \mathcal{G}_{k}(\Theta)$ such that $\mathcal{D}_{3,r}(G'_*,G_*)=2 \varepsilon$ and $\|f_{G'_*} - f_{G_*}\|_{L^2(\mu)} \leq C_1\varepsilon$. From Le Cam's lemma~\cite{yu97lecam}, as the Voronoi loss function $\mathcal{D}_{3,r}$ satisfies the weak triangle inequality, we obtain that
\begin{align}
    \inf_{\overline{G}_n\in\mathcal{G}_{k}(\Theta)}\sup_{G\in\mathcal{G}_{k}(\Theta)\setminus\mathcal{G}_{k_*-1}(\Theta)}\bbE_{f_{G}}[\mathcal{D}_{3,r}(\overline{G}_n,G)] & \gtrsim \frac{\mathcal{D}_{3,r}(G'_*,G_*)}{8} \text{exp}(- n \mathbb{E}_{X \sim \mu}[\text{KL}(\mathcal{N}(f_{G'_{*}}(x), \sigma^2),\mathcal{N}(f_{G_{*}}(x), \sigma^2))]) \nonumber \\
    & \gtrsim \varepsilon \cdot \text{exp}(-n \|f_{G'_*} - f_{G_*}\|_{L^2(\mu)}^2), \nonumber \\
    & \gtrsim \varepsilon \cdot \text{exp}(-C_{1} n \varepsilon^2), \label{eq:LeCam_inequality}
\end{align}
where the second inequality is due to the fact that $\text{KL}(\mathcal{N}(f_{G'_{*}}(x), \sigma^2),\mathcal{N}(f_{G_{*}}(x), \sigma^2)) = \dfrac{(f_{G'_*}(x) - f_{G_*}(x))^2}{2 \sigma^2}$. 

By choosing $\varepsilon=n^{-1/2}$, we obtain that $\varepsilon \cdot \text{exp}(-C_{1} n \varepsilon^2)=n^{-1/2}\exp(-C_1)$. As a consequence, we achieve the desired minimax lower bound in equation~\eqref{eq:minimax_activation_experts}.
\subsection{Proof of Proposition~\ref{prop:linear_limit}}
\label{appendix:linear_limit}
We need to prove that the following limit holds true for any $r\geq 1$:
\begin{align}
    \label{eq:ratio_zero_limit}
    \lim_{\varepsilon\to0}\inf_{G\in\mathcal{G}_{k}(\Theta):\mathcal{D}_{3,r}(G,G_*)\leq\varepsilon}\frac{\normf{f_{G}-f_{G_*}}}{\mathcal{D}_{3,r}(G,G_*)}=0.
\end{align}
For that purpose, it suffices to build a sequence of mixing measures $(G_n)$ such that both $\mathcal{D}_{3,r}(G_n,G_*)\to0$ and 
\begin{align*}
    \frac{\normf{f_{G_n}-f_{G_*}}}{\mathcal{D}_{3,r}(G_n,G_*)}\to0,
\end{align*}
as $n\to\infty$. To this end, we consider the sequence  $G_n=\sum_{i=1}^{k_*+1}\exp(\bzin)\delta_{(\boin,\ain,\bin)}$, where 
\begin{itemize}
    \item $\exp(\beta^n_{01})=\exp(\beta^n_{02})=\frac{1}{2}\exp(\beta^*_{01})+\frac{1}{2n^{r+1}}$ and  $\exp(\beta^n_{0i})=\exp(\beta^n_{0(i-1)})$ for any $3\leq i\leq k_*+1$;
    \item $\beta^n_{11}=\beta^n_{12}=\beta^*_{11}$ and  $\beta^n_{1i}=\beta^n_{1(i-1)}$ for any $3\leq i\leq k_*+1$;
    \item $a^n_1=a^n_2=a^*_1$ and $a^n_i=a^n_{i-1}$ for any $3\leq i\leq k_*+1$;
    \item $b^n_1=b^*_1+\frac{1}{n}$, $b^n_2=b^*_1-\frac{1}{n}$ and  $b^n_{i}=b^*_{i-1}$ for any $3\leq i\leq k_*+1$.
\end{itemize}
As a result, the loss function $\mathcal{D}_{3,r}(G_n,G_*)$ is reduced to
\begin{align}
    \label{eq:D_r_formulation}
    \mathcal{D}_{3,r}(G_n,G_*)=\frac{1}{n^{r+1}}+\Big[\exp(\beta^*_{01})+\frac{1}{n^{r+1}}\Big]\cdot\frac{1}{n^r}=\mathcal{O}(n^{-r}).
\end{align}
which indicates indicates that $\mathcal{D}_{3,r}(G_n,G_*)\to0$ as $n\to\infty$. Now, we prove that $\normf{f_{G_n}-f_{G_*}}/\mathcal{D}_{3,r}(G_n,G_*)\to0$. For that purpose, let us consider the quantity $Q_n(x):=[\sum_{j=1}^{k_*}\exp((\beta^*_{1j})^{\top}x+\beta^*_{0j})]\cdot[f_{G_n}(x)-f_{G_*}(x)]$. Then, we decompose $Q_n(x)$ as follows:
\begin{align*}
    Q_n(x)&=\sum_{j=1}^{k_*}\sum_{i\in\mathcal{A}_j}\exp(\beta^n_{0i})\Big[\exp((\boin)^{\top}x)((\ain)^{\top}x+\bin)-\exp((\boj)^{\top}x)((\aj)^{\top}x+\bj)\Big]\\
    &-\sum_{j=1}^{k_*}\sum_{i\in\mathcal{A}_j}\exp(\beta^n_{0i})\Big[\exp((\boin)^{\top}x)f_{G_n}(x)-\exp((\boj)^{\top}x)f_{G_n}(x)\Big]\\
    &+\sum_{j=1}^{k_*}\Big(\sum_{i\in\mathcal{A}_j}\exp(\bzin)-\exp(\bzj)\Big)\Big[\exp((\boj)^{\top}x)((\aj)^{\top}x+\bj)-\exp((\boj)^{\top}x)f_{G_n}(x)\Big]\\
    &:=A_n(x)-B_n(x)+E_n(x).
\end{align*}
From the definitions of $\beta^n_{1i},a^n_i$ and $b^n_i$, we can rewrite $A_n(x)$ as follows:
\begin{align*}
    A_n(x)=\sum_{i=1}^{2}\frac{1}{2}\exp(\beta^n_{01})\exp((\beta^*_{11})^{\top}x)(b^n_i-b^*_1)=\frac{1}{2}\exp(\beta^n_{01})\exp((\beta^*_{11})^{\top}x)[(b^n_1-b^*_1)+(b^n_2-b^*_1)]=0.
\end{align*}
Additionally, it can also be checked that $B_n(x)=0$. Next, we have $E_n(x)=\mathcal{O}(n^{-(r+1)})$, therefore, it follows that $E_n(x)/\mathcal{D}_{3,r}(G_n,G_*)\to0$. As a consequence, $Q_n(x)/\mathcal{D}_{3,r}(G_n,G_*)\to0$ as $n\to\infty$ for almost every $x$. Since the term $\sum_{j=1}^{k_*}\exp((\boj)^{\top}x+\bzj)$ is bounded, we deduce that $[f_{G_n}(x)-f_{G_*}(x)]/\mathcal{D}_{3,r}\to0$ for almost every $x$. This result indicates that $\normf{f_{G_n}-f_{G_*}}/\mathcal{D}_{3,r}\to0$ as $n\to\infty$. Hence, the proof of claim~\eqref{eq:ratio_zero_limit} is completed.

\subsection{Proof of Theorem~\ref{theorem:linear_experts}}
\label{appendix:linear_experts}
By leveraging the result of Proposition~\ref{prop:linear_limit} and the arguments for Theorem~\ref{theorem:singular_activation_experts} in Appendix~\ref{appendix:singular_activation_experts}, we achieve the following minimax lower bound for any $r\geq 1$:
\begin{align}
    \label{eq:minimax_linear_experts}
    \inf_{\overline{G}_n\in\mathcal{G}_{k}(\Theta)}\sup_{G\in\mathcal{G}_{k}(\Theta)\setminus\mathcal{G}_{k_*-1}(\Theta)}\bbE_{f_{G}}[\mathcal{D}_{3,r}(\overline{G}_n,G)]\gtrsim n^{-1/2}.
\end{align}

\section{Identifiability of the Softmax Gating Mixture of Experts}
\label{appendix:identifiability}
\begin{proposition}
    \label{prop:general_identifiability}
    If $f_{G}(x)=f_{G_*}(x)$ holds true for almost every $x$, then we get that $G\equiv G'$.
\end{proposition}
\begin{proof}[Proof of Proposition~\ref{prop:general_identifiability}]
    Since $f_{G}(x)=f_{G_*}(x)$ for almost every $x$, we have
    \begin{align}
        \label{eq:general_identifiable_equation}
        \sum_{i=1}^{k}\softmax\Big((\beta_{1i})^{\top}x+\beta_{0i}\Big)\cdot h(x,\eta_i)=\sum_{i=1}^{k_*}\softmax\Big((\boi)^{\top}x+\bzi\Big)\cdot h(x,\eta^*_i).
    \end{align}
    As the expert function $h$ satisfies the conditions in Definition~\ref{def:general_conditions}, the set $\{h(x,\eta'_i):i\in[k']\}$, where $\eta'_1,\ldots,\eta'_{k'}$ are distinct vectors for some $k'\in\mathbb{N}$, is linearly independent. If $k\neq k_*$, then there exists some $i\in[k]$ such that $\eta_i\neq\eta^*_j$ for any $j\in[k_*]$. This implies that $\softmax((\beta_{1i})^{\top}x+\beta_{0i})=0$, which is a contradiction. Thus, we must have that $k=k_*$. As a result, 
    \begin{align*}
        \Big\{\softmax\Big((\beta_{1i})^{\top}x+\beta_{0i}\Big):i\in[k]\Big\}=\Big\{\softmax\Big((\boi)^{\top}x+\bzi\Big):i\in[k_*]\Big\},
    \end{align*}
    for almost every $x$. WLOG, we may assume that 
    \begin{align}
        \label{eq:general_soft-soft}
        \softmax\Big((\beta_{1i})^{\top}x+\beta_{0i}\Big)=\softmax\Big((\boi)^{\top}x+\bzi\Big),
    \end{align}
    for almost every $x$ for any $i\in[k_*]$. It is worth noting that the $\softmax$ function is invariant to translations, then equation~\eqref{eq:general_soft-soft} indicates that $\beta_{1i}=\boi+v_1$ and $\beta_{0i}=\bzi+v_0$ for some $v_1\in\mathbb{R}^d$ and $v_0\in\mathbb{R}$. However, from the assumptions $\beta_{1k}=\beta^*_{1k}=\zerod$ and $\beta_{0k}=\beta^*_{0k}=0$, we deduce that $v_1=\zerod$ and $v_0=0$. Consequently, we get that $\beta_{1i}=\boi$ and $\beta_{0i}=\bzi$ for any $i\in[k_*]$. Then, equation~\eqref{eq:general_identifiable_equation} can be rewritten as
    \begin{align}
        \label{eq:general_new_identifiable_equation}
        \sum_{i=1}^{k_*}\exp(\beta_{0i})\exp\Big((\beta_{1i})^{\top}x\Big)h(x,\eta_i)=\sum_{i=1}^{k_*}\exp(\bzi)\exp\Big((\boi)^{\top}x\Big)h(x,\eta^*_i),
    \end{align}
    for almost every $x$. Next, we denote $P_1,P_2,\ldots,P_m$ as a partition of the index set $[k_*]$, where $m\leq k$, such that $\exp(\beta_{0i})=\exp(\beta^*_{0i'})$ for any $i,i'\in P_j$ and $j\in[k_*]$. On the other hand, when $i$ and $i'$ do not belong to the same set $P_j$, we let $\exp(\beta_{0i})\neq\exp(\beta_{0i'})$. Thus, we can reformulate equation~\eqref{eq:general_new_identifiable_equation} as
    \begin{align*}
        \sum_{j=1}^{m}\sum_{i\in{P}_j}\exp(\beta_{0i})\exp\Big((\beta_{1i})^{\top}x\Big)h(x,\eta_i)=\sum_{j=1}^{m}\sum_{i\in{P}_j}\exp(\bzi)\exp\Big((\boi)^{\top}x\Big)h(x,\eta^*_i),
    \end{align*}
    for almost every $x$. Recall that $\beta_{1i}=\boi$ and $\beta_{0i}=\bzi$ for any $i\in[k_*]$, then the above leads to
    \begin{align*}
        \{\eta_i:i\in P_j\}\equiv\{\eta^*_i:i\in P_j\},
    \end{align*}
    for almost every $x$ for any $j\in[m]$. 
    As a consequence, 
    \begin{align*}
        G=\sum_{j=1}^{m}\sum_{i\in P_j}\exp(\beta_{0i})\delta_{(\beta_{1i},\eta_i)}=\sum_{j=1}^{m}\sum_{i\in P_j}\exp(\beta_{0i})\delta_{(\boi,\eta^*_i)}=G_*.
    \end{align*}
    Hence, we reach the conclusion of this proposition.
\end{proof}

\section{Numerical Experiments}
\label{sec:experiments}
In this section, we conduct a simulation study to empirically demonstrate that the convergence rates of least square estimation under the softmax gating MoE model with ridge experts $h_1(x,(a,b))=\mathrm{sigmoid}(ax+b)$ are significantly faster than those obtained when using linear experts $h_2(x,(a,b))=ax+b$. We conduct those experiments under both the exact-specified setting (\emph{when the true number of experts $k_*$ is known}) and the over-specified setting (\emph{when the true number of experts $k_*$ is unknown}).

\textbf{Synthetic Data.} 
First, we assume that the true mixing measure $G_*=\sum_{i=1}^{k_*}\exp(\beta^*_{0i})\delta_{(\beta^*_{1i},a^*_i,b^*_i)}$ is of order $k_* = 2$ and associated with the following ground-truth parameters: 

$$\beta^*_{01} = 0.0, \quad  \beta^*_{11} = 1.0, \quad a^*_1 = -1.0, \quad b_1^* = 2.0,$$

$$\beta^*_{02} = 0.0, \quad \beta^*_{12} = 0.0, \quad a^*_2 = 1.0, \quad b_2^* = 2.0.$$

Then, we generate i.i.d samples $\{(X_i, Y_i)\}_{i=1}^n$ by first sampling $X_i$'s from the uniform distribution $\mathrm{Uniform}[0, 1]$ and then sampling $Y_i$'s from the regression equation

$$Y_i=f_{G_*}(X_i)+\varepsilon_i,$$

where $\varepsilon_1,\ldots,\varepsilon_n$ are independent Gaussian noise variables such that $\mathbb{E}[\varepsilon_i|X_i]=0$ and $\var(\varepsilon_i|X_i)=1$.

\textbf{Initialization.} 
For each $k\in\{k_*,k_*+1\}$, we randomly distribute elements of the set $\{1, 2, ..., k\}$ into $k_*$ different Voronoi cells $\mathcal{A}_1, \mathcal{A}_2, \ldots, \mathcal{A}_{k_*}$, each contains at least one element. Moreover, we repeat this process for each replication. Subsequently, for each $j\in[k_*]$, we initialize parameters $\beta_{1i}$ by sampling from a Gaussian distribution centered around its true counterpart $\beta^*_{1j}$ with a small variance, where $i\in \mathcal{A}_j$. Other parameters $\beta_{0i},a_i,b_i$ are also initialized similarly.

\textbf{Training.} We use the stochastic gradient descent algorithm to minimize the mean square losses. We conduct 20 sample generations for each configuration, across a spectrum of 20 different sample sizes $n$ ranging from $10^4$ to $10^5$. Finally, we generate log-log scaled plots for the Voronoi loss functions. For ridge experts, we use the Voronoi loss $\mathcal{D}_2$ given in Section~\ref{sec:input-dependent_experts}, while for linear experts, we use the Voronoi loss $\mathcal{D}_{3,r}$ in Section~\ref{sec:polynomial_experts}.

\begin{itemize}
    \item \textit{Exact-specified setting:} Under this setting, as the true number of experts $k_*$ is known, we set $k=k_*=2$.
    \item \textit{Over-specified setting:} Under this setting, as $k_*$ is unknown, we over-specified the true model by setting $k=3$.
\end{itemize}

\textbf{Remark.} From Figure~\ref{fig:linear}, it can be seen that under the exact-specified and over-specified settings, the convergence rates of least square estimators $\widehat{G}_n$ when using linear experts are significantly slow, at orders $\mathcal{O}(n^{-0.06})$ and $\mathcal{O}(n^{-0.04})$, respectively. This observation totally aligns with our theoretical result in Theorem~\ref{theorem:linear_experts}.

On the other hand, Figure~\ref{fig:sigmoid} indicates that when using ridge experts, the least square estimator $\widehat{G}_n$ converges to $G_*$ at much faster rates, at order $\mathcal{O}(n^{-0.54})$ under the exact-specified setting, and at order $\mathcal{O}(n^{-0.57})$ under the over-specified setting. These empirical rates match the theoretical rate $\mathcal{O}(n^{-0.5})$ captured in Theorem~\ref{theorem:activation_experts}.

\begin{figure}[!ht]
    \centering
    \begin{subfigure}{.45\textwidth}
        \centering
        \includegraphics[scale = .5]{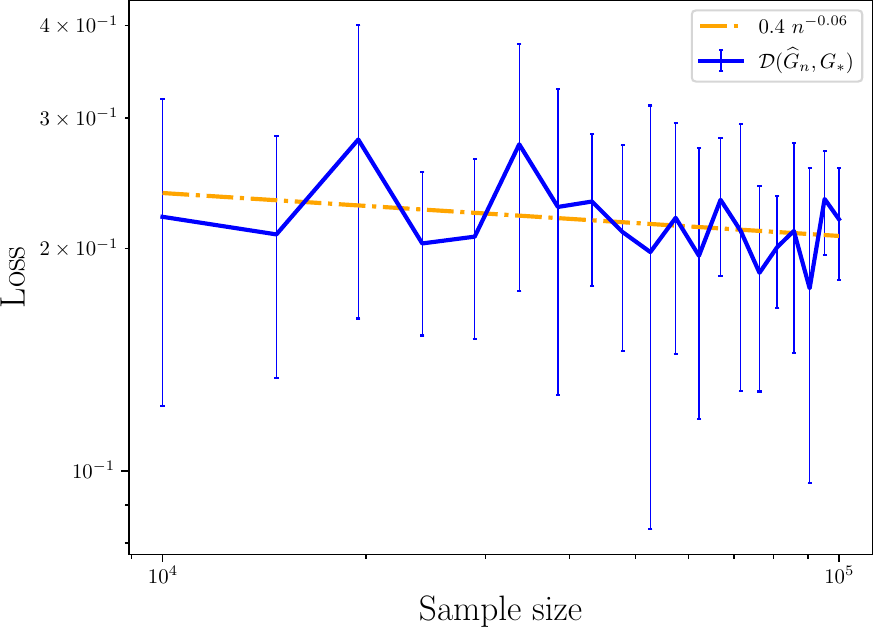}
        \caption{Exact-specified setting}	
        \label{fig:linear_exact}
    \end{subfigure}
    \hspace{0.6cm}
    \begin{subfigure}{.45\textwidth}
        \centering
        \includegraphics[scale = .5]{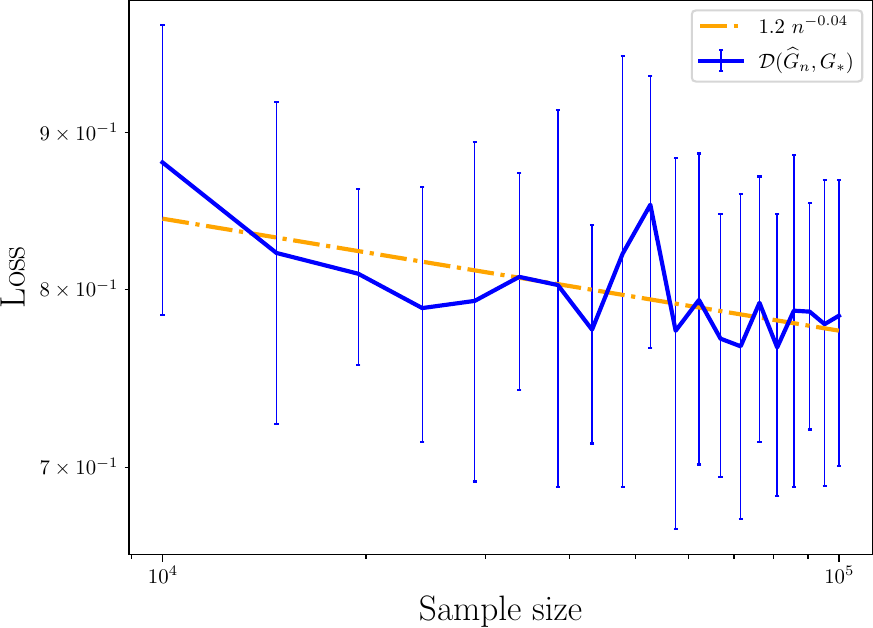}
        \caption{Over-specified setting}	
        \label{fig:linear_over}
    \end{subfigure}%
    \caption{Log-log scaled plots illustrating empirical convergence rates of parameter estimation in the softmax gating mixture of \textbf{linear experts} under the exact-specified setting (Figure~\ref{fig:linear_exact}) and the over-specified setting (Figure~\ref{fig:linear_over}). The blue curves depict the mean discrepancy between the least squares estimator $\widehat{G}_n$ and the true mixing measure $G_*$ under the loss $\mathcal{D}_{3,r}$, accompanied by error bars signifying two empirical standard deviations. Additionally, an orange dash-dotted line represents the least-squares fitted linear regression line for these data points.}
    \label{fig:linear}
\end{figure}

\begin{figure}[!ht]
    \centering
    \begin{subfigure}{.45\textwidth}
        \centering
        \includegraphics[scale = .5]{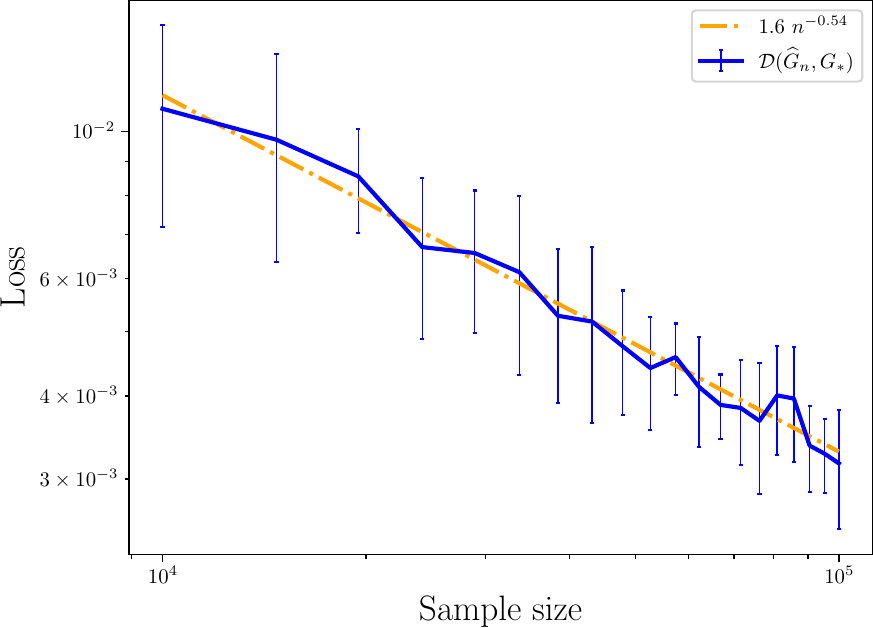}
        \caption{Exact-specified setting}	
        \label{fig:sigmoid_exact}
    \end{subfigure}
    \hspace{0.6cm}
    \begin{subfigure}{.45\textwidth}
        \centering
        \includegraphics[scale = .5]{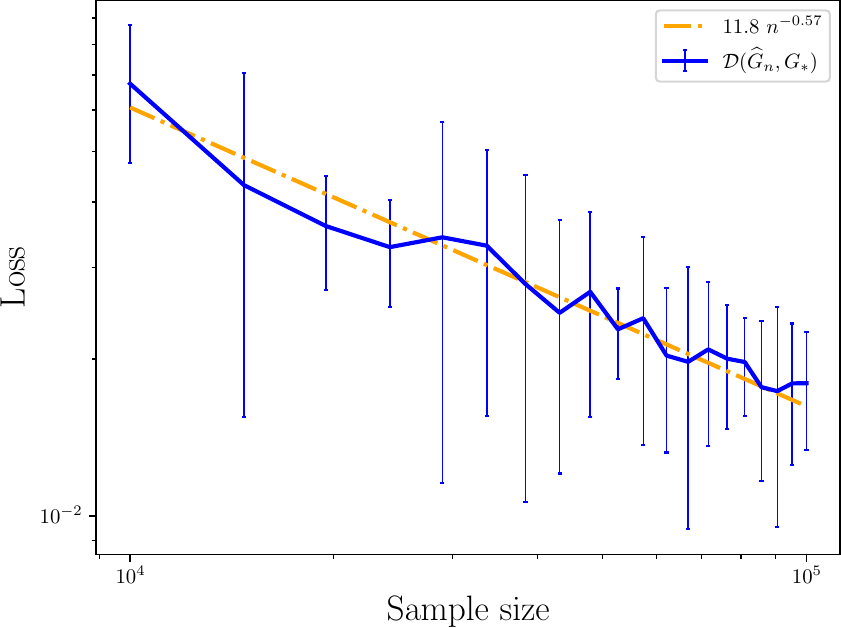}
        \caption{Over-specified setting}	
        \label{fig:sigmoid_over}
    \end{subfigure}%
    \caption{Log-log scaled plots illustrating empirical convergence rates of parameter estimation in the softmax gating mixture of \textbf{ridge experts} with the sigmoid activation under the exact-specified setting (Figure~\ref{fig:sigmoid_exact}) and the over-specified setting (Figure~\ref{fig:sigmoid_over}). The blue curves depict the mean discrepancy between the least squares estimator $\widehat{G}_n$ and the true mixing measure $G_*$ under the loss $\mathcal{D}_{2}$, accompanied by error bars signifying two empirical standard deviations. Additionally, an orange dash-dotted line represents the least-squares fitted linear regression line for these data points.}
    \label{fig:sigmoid}
\end{figure}

\end{document}